\documentclass{article}

\usepackage{microtype}
\usepackage{graphicx}
\usepackage{booktabs} 

\usepackage{hyperref}

\usepackage[accepted]{icml2023}

\usepackage{amsmath}
\usepackage{amssymb}
\usepackage{mathtools}
\usepackage{amsthm}

\usepackage[capitalize,noabbrev]{cleveref}

\theoremstyle{plain}
\newtheorem{theorem}{Theorem}[section]
\newtheorem{proposition}[theorem]{Proposition}
\newtheorem{lemma}[theorem]{Lemma}
\newtheorem{corollary}[theorem]{Corollary}

\theoremstyle{definition}
\newtheorem{definition}[theorem]{Definition}
\newtheorem{assumption}[theorem]{Assumption}

\theoremstyle{remark}
\newtheorem{remark}[theorem]{Remark}
\newtheorem{example}[theorem]{Example}

\usepackage[textsize=tiny]{todonotes}

\icmltitlerunning{Identifiability and Generalizability in Constrained Inverse Reinforcement Learning}

\usepackage{enumerate}
\usepackage{bbm}
\usepackage{bm}
\usepackage{amsfonts}
\usepackage{mathtools}
\usepackage{caption}
\usepackage{subcaption}

\usepackage[page,header]{appendix}
\usepackage{titletoc}

\newcommand{\bc}[1]{\left\{{#1}\right\}}
\newcommand{\br}[1]{\left({#1}\right)}
\newcommand{\bs}[1]{\left[{#1}\right]}
\newcommand{\abs}[1]{\left| {#1} \right|}
\newcommand{\myvec}[1]{\ensuremath{\begin{bmatrix}#1\end{bmatrix}}}
\newcommand{\f}[1]{\bm{#1}}
\newcommand\numberthis{\addtocounter{equation}{1}\tag{\theequation}}
\newcommand{\norm}[1]{\left\lVert#1\right\rVert}
\newcommand{\R}{\mathbb{R}}
\newcommand{\N}{\mathbb{N}}
\newcommand{\E}{\mathbb{E}}

\newcommand{\Pp}{\f{\mathbbm{P}}}

\newcommand{\defeq}{\vcentcolon=}

\newcommand{\fsigma}{\f{\sigma}}
\newcommand{\fsigmaEhat}{\hat{\fsigma}^{\text{E}}_{\mathcal{D}}}
\newcommand{\fsigmaE}{\fsigma^{\text{E}}}
\newcommand{\fsigmaET}{\fsigma^{\text{E}}_T}
\newcommand{\fnu}{\f{\nu}}
\newcommand{\fmu}{\f{\mu}}

\newcommand{\fmuE}{\fmu^{\text{E}}}
\newcommand{\fmuEhat}{\hat{\fmu}^{\text{E}}_{\mathcal{D}}}
\newcommand{\fmuhat}{\hat{\fmu}}

\newcommand{\fPhi}{\f{\Phi}}
\newcommand{\fPsi}{\f{\Psi}}

\newcommand{\feta}{\f{\eta}}
\newcommand{\fxi}{\f{\xi}}
\newcommand{\fXi}{\f{\Xi}}

\newcommand{\fw}{\f{w}}
\newcommand{\fwhat}{\hat{\f{w}}}
\newcommand{\fd}{\f{d}}
\newcommand{\fp}{\f{p}}
\newcommand{\fq}{\f{q}}
\newcommand{\fb}{\f{b}}

\newcommand{\fx}{\f{x}}

\newcommand{\fy}{\f{y}}
\newcommand{\fz}{\f{z}}
\newcommand{\id}{\f{I}}
\newcommand{\ones}{\f{1}}
\newcommand{\fr}{\f{r}}

\newcommand{\fD}{\f{D}}
\newcommand{\fP}{\f{P}}
\newcommand{\fE}{\f{E}}
\newcommand{\fA}{\f{A}}
\newcommand{\fB}{\f{B}}
\newcommand{\fe}{\f{e}}
\newcommand{\fpi}{\f{\pi}}
\newcommand{\fpihat}{\hat{\fpi}}
\newcommand{\fpiE}{\fpi^{\text{E}}}
\newcommand{\piE}{\pi^{\text{E}}}

\newcommand{\Lhat}{\hat{L}}
\newcommand{\frE}{\fr^{\text{E}}}

\newcommand{\frhat}{\hat{\fr}}
\newcommand{\DKL}{D_{\text{KL}}}
\DeclareMathOperator{\interior}{\mathbf{int}}
\DeclareMathOperator{\relint}{\mathbf{relint}}
\DeclareMathOperator{\relbd}{\mathbf{relbd}}
\DeclareMathOperator{\dom}{\mathbf{dom}}
\DeclareMathOperator{\rank}{\mathbf{rank}}
\DeclareMathOperator*{\argmin}{\arg\!\min}
\DeclareMathOperator*{\argmax}{\arg\!\max}

\DeclareMathOperator{\spn}{\mathbf{span}}
\DeclareMathOperator{\cone}{\mathbf{cone}}
\newcommand{\RL}{\mathsf{RL}_{\mathcal{F}}}
\newcommand{\RLO}{\mathsf{RL}_{\mathcal{M}}}
\newcommand{\IRL}{\mathsf{IRL}_{\mathcal{R}, \mathcal{F}}}
\newcommand{\IRLO}{\mathsf{IRL}_{\mathcal{F}}}

\newcommand{\ceil}[1]{\left \lceil #1 \right \rceil }

\newcommand{\headSpaceBefore}{\vspace{-0.0cm}} 
\newcommand{\headSpaceAfter}{\vspace{-0.0cm}} 

\newcommand{\parSpace}{\vspace{-0.0cm}} 

\begin{document}

\twocolumn[
\icmltitle{Identifiability and Generalizability in \\Constrained Inverse Reinforcement Learning}



\icmlsetsymbol{equal}{*}

\begin{icmlauthorlist}
\icmlauthor{Andreas Schlaginhaufen}{EPFL}
\icmlauthor{Maryam Kamgarpour}{EPFL}
\end{icmlauthorlist}

\icmlaffiliation{EPFL}{SYCAMORE Lab, École Polytechnique Fédérale de Lausanne (EPFL), 1015 Lausanne, Switzerland}

\icmlcorrespondingauthor{Andreas Schlaginhaufen}{andreas.schlaginhaufen@epfl.ch}

\icmlkeywords{Machine Learning, ICML}

\vskip 0.3in
]



\printAffiliationsAndNotice{}  
\begin{abstract}
Two main challenges in Reinforcement Learning (RL) are designing appropriate reward functions and ensuring the safety of the learned policy. To address these challenges, we present a theoretical framework for Inverse Reinforcement Learning (IRL) in constrained Markov decision processes. From a convex-analytic perspective, we extend prior results on reward identifiability and generalizability to both the constrained setting and a more general class of regularizations. In particular, we show that identifiability up to potential shaping \citep{cao2021identifiability} is a consequence of entropy regularization and may generally no longer hold for other regularizations or in the presence of safety constraints. We also show that to ensure generalizability to new transition laws and constraints, the true reward must be identified up to a constant. Additionally, we derive a finite sample guarantee for the suboptimality of the learned rewards, and validate our results in a gridworld environment.
\end{abstract}
\headSpaceBefore\vspace{-0.1cm}
\section{Introduction}
\headSpaceAfter
Reinforcement Learning (RL) has been successfully applied to many artificial intelligence tasks such as playing the games of Chess, Go, and Starcraft \citep{silver2016mastering}, control of humanoid robots \citep{akkaya2019solving}, or fine-tuning of large language models \citep{stiennon2020learning}. However, two of the key challenges in bringing RL to the real world include designing suitable reward functions for the problem at hand and guaranteeing safety of the learned policy.\looseness-1

While a surge of recent work addresses safe RL, past work has focused on the case in which the reward function is known. In many cases, such as autonomous driving or human-robot interactions, the reward functions are a priori unknown and are hard to design. The framework of inverse reinforcement learning (IRL) addresses learning reward functions based on expert demonstrations. However, most past work on IRL is not incorporating safety explicitly. While safety can be learned implicitly by learning a reward that penalizes unsafe behavior, \emph{incorporating safety in IRL is essential in ensuring that the learned rewards can be transferred to new tasks efficiently}. 
\vspace{-0.1cm}
\paragraph{Related Work}
Safe RL for Markov Decision Processes (MDPs) has been extensively studied and various notions of safety have been considered \citep{garcia2015comprehensive}. A well-established framework for safety is Constrained Markov Decision Processes (CMDPs) \citep{altman1999constrained}. In this setting, safety is defined through a set of cost functions on the state and actions, whose expectations should be bounded by a predefined threshold. The CMDP approach has been broadly applied to safe RL \citep{achiam2017constrained, chow2018lyapunov, turchetta2020safe}.

The above safe RL works all assume that the reward is known or can be evaluated along the MDP trajectories. Learning rewards from a data set of expert demonstrations -- i.e. the concept of IRL -- was first introduced by \citet{russell1998learning}. Whereas imitation learning \citep{pomerleau1988alvinn, syed2007game, ho2016generative, garg2021iq} attempts to directly recover the expert's policy, the goal of IRL is to learn the expert's latent reward function. The main motivation for IRL is that the reward, being independent of the transition law, provides the most succinct and transferable description of a task \citep{ng2000algorithms}. 

A fundamental challenge in IRL is that in general many rewards can generate a given optimal behavior, making it difficult to recover the true underlying reward function. In particular, \citet{ng1999policy} show that the set of optimal policies is always invariant under the so-called potential shaping transformations. Additionally, the non-uniqueness of the optimal policy corresponding to some reward can lead to trivial solutions to the IRL problem. For instance, for a constant reward all policies are optimal, but such a reward is most likely not informative for the expert's task. Various methods have been proposed to deal with the above two degeneracies. These include approaches based on margin maximization \citep{abbeel2004apprenticeship, ratliff2006maximum} or Bayesian reasoning \citep{ramachandran2007bayesian, choi2012nonparametric}. One promising approach is Maximum Causal Entropy IRL (MCE-IRL) \citep{ziebart2010modeling, zhou2017infinite}, which ensures uniqueness of the optimal policy via entropy regularization and has led to state-of-the-art imitation learning methods \citep{ho2016generative, garg2021iq}. MCE-IRL provably recovers the expert's true reward up to potential shaping transformations \citet{cao2021identifiability}. Moreover, the MCE-IRL framework has been extended to more general policy regularizations \citet{jeon2020regularized}, but the question of identifiability up to potential shaping transformations remains unanswered in this more general setting. 

While both IRL and safe RL problems have been studied extensively, less work has focused on safety aspects in IRL. Robustness to transition laws has been addressed by \citep{fu2017learning, viano2021robust}. Motivated by safety and risk considerations, \citet{majumdar2017risk} suggest that the expert may not be simply minimizing an expected cumulative cost (corresponding to a constraint), but some general coherent risk measure of that cost. Moreover, \citet{tschiatschek2019learner} provide a learner-aware MCE-IRL framework in which the learner has their own preference constraints, such as safety. Most recently, \citet{malik2021inverse} assume a known reward and address the problem of learning only the constraints from demonstrations, whereas \citet{ding2022x} consider IRL with combinatorial constraints. These recent works incorporate safety in various approaches but do not focus on identifiability and generalizability aspects. 
\vspace{-0.1cm}
\paragraph{Contribution}
Approaching the CMDP problem as a linear program in the so-called occupancy measure, we present a constrained IRL framework for arbitrary convex regularizations of the occupancy measure. In this general setting, we then address the questions of identifiability and generalizability to new transition laws and constraints. In particular, we first show that arbitrary (strictly) convex policy regularizations \citep{geist2019theory} yield a (strictly) convex regularization of the occupancy measure, and are hence naturally incorporated in our framework (Proposition~\ref{prop:convexity}). Our first main result (Theorem~\ref{thm:identifiability}) is a complete characterization of the set of rewards for which a given expert is optimal. Notably, we show that identifiability up to potential shaping transformations in MCE-IRL is a consequence of entropy regularizations and generally no longer holds for other regularizations (such as e.g. sparse Tsallis entropy \citep{lee2018maximum}), nor in the constrained setting. Our second main result (Theorem~\ref{thm:intersection}) shows that generalizability to new transition laws and safety constraints is only possible if the expert's reward is recovered up to a constant. Furthermore, we provide a verifiable sufficient condition for generalizability.  Our proof techniques, based on convex analysis, unify and generalize past work. In Section~\ref{sec:finite_sample}, we address the finite sample setting and provide a novel result for the number of expert demonstrations needed to recover a reward whose optimal policy is close to the expert's policy. In Section~\ref{sec:experiments}, we experimentally verify our results in a gridworld environment.\looseness-1
\headSpaceBefore
\section{Background}\label{sec:background}
\vspace{-0.1cm}
\paragraph{Notation}
We use $\N, \R$, and $\R_+$ to denote the set of natural, real, and non-negative real numbers, respectively. For an arbitrary set $\mathcal{X}$ we denote $2^{\mathcal{X}}$ for the set of all subsets of $\mathcal{X}$, and for a finite set $\mathcal{Y}$ we denote $\Delta_{\mathcal{Y}}$ for the probability simplex over $\mathcal{Y}$. For two vectors $\f{a}, \f{b}\in\R^l$ we notate $\f{a}\leq\f{b}$ for the element-wise comparison. Furthermore, we denote $\id_l$ for the identity matrix in $\R^l$, we let $\ones_l\in\R^l$ be the all-one vector, and $\norm{\cdot}$ indicates a general norm on $\R^l$. We also frequently use $\norm{\cdot}_p$ with $p\in\bc{1,2, \infty}$ for the $p$-norms. Given two matrices $\fA, \fB$ with compatible dimensions we denote $\myvec{\fA & \fB}$ for their concatenation. For a set of vectors $\bc{\f{v}_1, \hdots, \f{v}_q}\subset\R^l$ we denote $\spn(\f{v}_1, \hdots, \f{v}_q)$ for their linear span, and we denote $\spn(\f{A})$ for the column span of some matrix $\f{A}\in\R^{l\times q}$. Similarly, we use $\cone(\f{v}_1, \hdots, \f{v}_q)\defeq\bc{\sum_{i=1}^q c_i \f{v}_i : c_i \geq 0}$ to denote the conic hull of $\bc{\f{v}_1, \hdots, \f{v}_q}$. Moreover, the Minkowski sum of two sets $\mathcal{X}, \mathcal{Y}$ is denoted as $\mathcal{X} + \mathcal{Y}\defeq\bc{x+y:x\in\mathcal{X}, y\in\mathcal{Y}}$ and as $x + \mathcal{Y}$ if $\mathcal{X}=\bc{x}$. It holds $\mathcal{X}+\emptyset = \emptyset$, as well as, $\spn \emptyset = \cone \emptyset = \bm 0$, where $\bm 0$ is the zero vector. For a set-valued mapping $g:\mathcal{X}\to 2^{\mathcal{Y}}$ we denote by $g(A)\defeq\bigcup_{x\in A} g(x)$ the image of $A\subseteq \mathcal{X}$ under $g$. We often encounter functions $\f{h}:\mathcal{S}\times\mathcal{A}\to\R^l$ with finite domain $\mathcal{S}\times\mathcal{A}=\bc{s_1, \hdots, s_n}\times\bc{a_1, \hdots, a_m}$. Here, we use the vector notation\vspace{-0.1cm}
\begin{align}\label{eq:vector_convention}
    \f{h} \defeq &\big[\f{h}(s_1, a_1),\f{h}(s_2, a_1),\hdots \f{h}(s_n, a_1), \\
    &\f{h}(s_1, a_2),\hdots, \f{h}(s_n, a_2), \hdots, \f{h}(s_n, a_m) ]^\top\in\R^{nm\times l}.\nonumber
\end{align}
Similarly, we identify functions $\mathcal{S}\to\R^l$ and $\mathcal{A}\to\R^l$ with matrices in $\R^{n\times l}$ and $\R^{m\times l}$. The interior $\interior \mathcal{X}$, the relative interior $\relint \mathcal{X}$, the relative boundary $\relbd \mathcal{X}$, and the normal cone $N_{\mathcal{X}}(\fx)$ of some set $\mathcal{X}\subseteq\R^l$; as well as the subdifferential $\partial g(\fx)$ of some function $g:\mathcal{X}\to \R$ are for completeness defined in Appendix~\ref{app:sec:notation}.
\vspace{-0.1cm}
\paragraph{Constrained Markov Decision Processes} 
We consider CMDPs \citep{altman1999constrained} defined by a tuple $M=\br{\mathcal{S},\mathcal{A}, \fP, \fnu_0, \fr, \fPsi, \fb, \gamma}$. Here, $\mathcal{S}$ and $\mathcal{A}$, with $\abs{\mathcal{S}}=n$ and $\abs{\mathcal{A}}=m>1$, denote the finite state and action spaces, $\fnu_0\in\Delta_{\mathcal{S}}$ the initial state distribution, $\fP:\mathcal{S}\times\mathcal{A}\to\Delta_{\mathcal{S}}$ a Markovian transition law, $\fr\in\R^{nm}$ a reward, and $\gamma\in(0,1)$ a discount factor. $\fPsi\defeq \myvec{\fPsi_1,\hdots, \fPsi_k}\in\R^{nm\times k}$ is a matrix of safety constraint costs and $\fb\in\R^k$ is the corresponding threshold. Starting from some initial state $s_0\sim\fnu_0$ the agent can at each step in time $t$, choose an action $a_t\in\mathcal{A}$, will arrive in some state $s_{t+1}\sim \fP(\cdot|s_t, a_t)$, and receives reward $\fr(s_t,a_t)$ and safety cost $\fPsi(s_t, a_t)$. The (regularized) CMDP problem is then defined as follows
\begin{alignat}{2}\label{eq:cmdp_policy}
    &\max_{\fpi\in\Pi}\;\; &&(1-\gamma)\E_{\fpi} \bs{\sum_{t=0}^\infty \gamma^t \bs{\fr(s_t,a_t)-\Omega\br{\fpi(\cdot|s_t)}}}\\
    &\text{s.t} &&(1-\gamma)\E_{\fpi} \bs{\sum_{t=0}^\infty \gamma^t \fPsi(s_t, a_t)}\leq \fb.\nonumber\vspace{-0.1cm}
\end{alignat}
Here, $\Pi:=\bc{\tilde{\fpi}:\mathcal{S}\to\Delta_{\mathcal{A}}}$ is the set of Markov policies, $\E_{\fpi}$ denotes the expectation with respect to the probability measure $\f{\Pp}_{\fnu_0}^{\fpi}$, induced by the initial state distribution $\fnu_0$ and the policy $a_t\sim \fpi(\cdot|s_t)$, on the sample space $(\mathcal{S}\times\mathcal{A})^\infty \defeq \bc{(s_0, a_0, s_1, a_1, \hdots): s_i\in\mathcal{S}, a_i\in\mathcal{A},\; i\in\N}$. The factor $(1-\gamma)$ is introduced for convenience. Furthermore, $\Omega:\Delta_{\mathcal{A}}\to\R$ is a convex regularization, which if strictly convex ensures uniqueness of the optimal policy \citep{geist2019theory}.

A widely used regularization is the negative Shannon entropy\footnote{Here, we use the convention $0\log 0 = 0$, which is standard in information theory \citep{cover1999elements} in order to continuously extend $H$ onto the non-negative orthant.} $\Omega = -\beta H$, with $\beta>0$ and
\begin{equation}\label{eq:entropy_reg}\vspace{-0.1cm}
    H:\Delta_{\mathcal{A}}\to\R_+, \; \fd\mapsto H(\fd) = -\sum_a \fd(a)\log\fd(a).\vspace{-0.1cm}
\end{equation}
For entropy regularization, the optimal policy can (under Assumption~\ref{ass:slater}) be shown to be non-vanishing (see Appendix~\ref{app:subsec:entropy_regularization}). Thus, it is often used to foster exploration during optimization \citep{neu2017unified, haarnoja2018soft}. Other regularizations such as the sparse Tsallis entropy \citep{lee2018sparse} lead to more sparse optimal policies. Next, we continue with a convex reformulation of problem~\eqref{eq:cmdp_policy}.\looseness-1
\vspace{-0.2cm}
\section{Convex Viewpoint}\label{sec:convex_viewpoint}
\vspace{-0.1cm}
\paragraph{Convex Reformulation}
While the optimization problem \eqref{eq:cmdp_policy} is in general non-convex \citep{agarwal2019reinforcement}, it admits a convex reformulation. To see this, we introduce the state-action occupancy measure $\fmu^{\fpi}\in\Delta_{\mathcal{S}\times\mathcal{A}}$ defined by\vspace{-0.1cm}
\begin{equation}
    \fmu^{\fpi}(s,a)\defeq (1-\gamma) \bs{\sum_{t=0}^\infty \gamma^t \f{\Pp}_{\fnu_0}^{\fpi}(s_t=s,a_t=a)}.\vspace{-0.1cm}
\end{equation}
For any function $\bm h:\mathcal{S}\times\mathcal{A}\to \R^l$ this allows us to rewrite $(1-\gamma)\E_{\fpi} \bs{\sum_{t=0}^\infty \gamma^t \bm h(s_t,a_t)} = \bm h^\top \fmu^{\fpi}$.
As shown by \citet{puterman1994markov}, the set of valid occupancy measures $\mathcal{M}\defeq\bc{\fmu^{\fpi}:\fpi\in\Pi}\subseteq \Delta_{\mathcal{S}\times\mathcal{A}}$ is characterized by the Bellman flow constraints\vspace{-0.1cm}
\begin{equation}\label{eq:bellman_flow}
    \mathcal{M} =  \bc{\fmu\in\R^{nm}_+ : (\fE  - \gamma \fP)^\top\fmu = (1-\gamma)\fnu_0},\vspace{-0.1cm}
\end{equation}
where $\fE \defeq \myvec{\id_n & \hdots & \id_n}^{\top}\in\R^{nm\times n}$, $\fP\in\R^{nm\times n}$ is the transition law in matrix form following our convention introduced in \eqref{eq:vector_convention}, and $(\fE^\top\fmu)(s) = \sum_{a}\fmu(s,a)$ is the so-called state occupancy measure. We refer to states with zero state occupancy measure as unvisited. 

As stated by \citet{puterman1994markov} there is a one-to-one mapping $T:\mathcal{M}\to\Pi, \fmu\mapsto\fpi^{\fmu}$ defined via\vspace{-0.1cm}
\begin{equation}\label{eq:pi_mu}
    \fpi^{\fmu}(a|s) \defeq \begin{cases} \fmu(s,a)/(\fE^\top\fmu)(s)&, (\fE^\top\fmu)(s)>0\\ 1/|\mathcal{A}|&, \text{ otherwise.} \end{cases}\vspace{-0.1cm}
\end{equation}
The choice $\fpi^{\fmu}(a|s)=1/|\mathcal{A}|$ for unvisited states is arbitrary. Without regularization, the above one-to-one mapping proves equivalence of the CMDP problem~\eqref{eq:cmdp_policy} to a linear program in the occupancy measure. Our next result shows that for arbitrary (strictly) convex regularizations $\Omega$, the objective of \eqref{eq:cmdp_policy} is still (strictly) convex in the occupancy measure.\looseness-1
\begin{proposition}\label{prop:convexity}
    Let $f(\fmu)=\E_{(s,a)\sim\fmu}\bs{\Omega\left(\fpi^{\fmu}(\cdot|s)\right)}$. \vspace{-0.4cm}
    \begin{enumerate}[(a)]
        \item If $\Omega$ is convex, then so is $f$.\vspace{-0.2cm}
        \item If $\Omega$ is strictly convex, then so is $f$.
    \end{enumerate}\vspace{-0.2cm}
\end{proposition}
\vspace{-0.2cm}
For the proof of Proposition~\ref{prop:convexity} we refer to Appendix~\ref{app:subsec:prop:convexity}. To the best of our knowledge, Proposition~\ref{prop:convexity} is a novel result that has previously only been shown for special cases such as the Shannon entropy \citep{ziebart2010modeling}.\footnote{\citet{jeon2020regularized} mention that strict convexity of a policy regularizer is not always guaranteeing strict convexity in the occupancy measure. However, Proposition~\ref{prop:convexity} shows the contrary.}

Due to the one-to-one mapping $T$ and Proposition~\ref{prop:convexity}, the regularized CMDP problem \eqref{eq:cmdp_policy} is equivalent to\vspace{-0.1cm}
\begin{equation}\label{eq:cmdp_occ}\tag{P}
    \max_{\fmu\in\mathcal{F}}  \;\; \fr^\top \fmu - f(\fmu),\vspace{-0.1cm}
\end{equation}
where $f$ is defined as in Proposition~\ref{prop:convexity} and the set of feasible occupancy measures is given by\vspace{-0.1cm}
\begin{equation}\label{eq:feasible_occ}
    \mathcal{F}\defeq\bc{\fmu\in\R^{nm}:\fmu\in\mathcal{M}, \;\fPsi^\top\fmu\leq \fb}\subseteq \Delta_{\mathcal{S}\times\mathcal{A}}.\vspace{-0.1cm}
\end{equation}
For the rest of the paper, we will be focusing on problem~\eqref{eq:cmdp_occ}. We let $f:\mathcal{X}\to \R$ be an arbitrary convex continuous regularization and $\mathcal{X}\subseteq\R^{nm}$ a closed convex set with $\Delta_{\mathcal{S}\times\mathcal{A}}\subseteq \mathcal{X}$. This formulation includes unregularized CMDPs \citep{altman1999constrained}, policy regularization \citep{geist2019theory}, as well as many other regularizations such as entropy regularization in the occupancy measure $f(\fmu) = -\beta H(\fmu)$.

If the feasible set $\mathcal{F}$ is non-empty and $f$ is strictly convex, it follows from compactness of $\mathcal{F}$ that problem \eqref{eq:cmdp_occ} admits a unique optimal solution. However, we will always state explicitly whether $f$ is assumed to be strictly convex or not. For the further analysis, it will be convenient to define the (set-valued) solution map $\RL:\R^{nm}\to 2^{\Delta_{\mathcal{S}\times\mathcal{A}}}$ via \vspace{-0.1cm}
\begin{equation}
    \RL(\fr)\defeq \argmax_{\fmu\in\mathcal{F}} \fr^\top\fmu - f(\fmu),\vspace{-0.1cm}
\end{equation}
and analogously $\RLO$ for the unconstrained MDP problem.
\paragraph{Strong Duality}
Next, we show that for strictly convex regularization, the CMDP problem \eqref{eq:cmdp_occ} is equivalent to an unconstrained MDP problem with a modified reward. To see this, we consider the Lagrangian dual problem of \eqref{eq:cmdp_occ} obtained via relaxation of the safety constraint\vspace{-0.1cm}
\begin{equation}\label{eq:cmdp_occ_dual}\tag{D}
    \min_{\fxi\geq\f{0}} \max_{\fmu\in\mathcal{M}} \fr^\top \fmu - f(\fmu) + \fxi^\top \br{\fb - \fPsi^\top \fmu}.\vspace{0.1cm}
\end{equation}
\begin{assumption}[Slater's condition]\label{ass:slater}
    Let $\relint\mathcal{F}\neq \emptyset$\footnote{Except for Theorem~\ref{thm:intersection}, our results continue to hold for the slightly weaker version of Slater's condition $\relint(\mathcal{X})\cap\mathcal{F}\neq \emptyset$, since the feasible set $\mathcal{F}$ is polyhedral.}.
\end{assumption}
\begin{assumption}[Strict convexity]\label{ass:strict_convexity}
    Let $f:\mathcal{X}\to \R$ be strictly convex.
\end{assumption}
Under Assumption~\ref{ass:slater} above, the optimal values of \eqref{eq:cmdp_occ} and \eqref{eq:cmdp_occ_dual} coincide and the dual optimum is attained \citep{altman1999constrained}. Under the additional assumption of strict convexity, we show that the CMDP problem \eqref{eq:cmdp_occ} is equivalent to an unconstrained MDP problem.
\begin{proposition}[Strong duality]\label{prop:strong_duality}
    If Assumption~\ref{ass:slater} and \ref{ass:strict_convexity} hold, the dual optimum of \eqref{eq:cmdp_occ_dual} is attained for some $\fxi^*\geq\bm 0$, and \eqref{eq:cmdp_occ} is equivalent to an unconstrained MDP problem of reward $\fr-\fPsi\fxi^*$. In other words, it holds
    \begin{equation}\label{eq:equivalence_constrained_unconstrained_MDP}
    \RL(\fr) = \RLO(\fr-\fPsi\fxi^*).
    \end{equation}
\end{proposition}
Note that without strict convexity, we have $\RL(\fr) \subseteq \RLO(\fr-\fPsi\fxi^*)$. The proof of Proposition~\ref{prop:strong_duality} and a simple counterexample of why \eqref{eq:equivalence_constrained_unconstrained_MDP} does not hold in the unregularized case are provided in Appendix~\ref{app:subsec:prop:strong_duality} and \ref{app:subsec:remarks_strong_duality}, respectively. As a consequence of \eqref{eq:equivalence_constrained_unconstrained_MDP}, we may be tempted to ignore safety constraints for IRL and recover the modified reward $\fr-\fPsi\fxi^*$ via standard unconstrained IRL. However, as we discuss before Section~\ref{sec:finite_sample}, such a reward is not guaranteed to generalize to different transition laws and safety constraints.\looseness-1
\headSpaceBefore
\section{Constrained IRL}\label{sec:constrained_IRL}
\headSpaceAfter
\paragraph{Problem Formulation}\label{par:problem_formulation}
Given a CMDP without reward $M\setminus\fr =\br{\mathcal{S},\mathcal{A}, \fP, \fnu_0, \fPsi, \fb, \gamma}$ and a data set of demonstrations $\mathcal{D} = \bc{(s^i_t, a^i_t)_{t=0}^T}_{i=1}^N$ from some expert $\fmuE$, constrained IRL aims to recover a reward from some reward class $\mathcal{R}\subseteq \R^{nm}$ for which the expert is optimal. Unless stated otherwise, we let the reward class $\mathcal{R}$ be an arbitrary convex set. Clearly, the above IRL problem only has a solution under the following realizability assumption.
\begin{assumption}[Realizability] \label{ass:realizability}
Assume the expert is optimal for some $\fr^\text{E}\in \mathcal{R}$ i.e. $\fmuE \in \RL(\fr^\text{E})$. 
\end{assumption}
Next, we show that the constrained IRL problem can be formulated as an optimization problem. We first consider an idealized setting, where we are given access to the expert's true occupancy measure $\fmuE$ rather than to the demonstrations $\mathcal{D}$, and address the finite sample setting in Section~\ref{sec:finite_sample}.
\parSpace
\paragraph{Min-Max Formulation}
It is well-known that in the absence of constraints, the IRL problem can be captured as a min-max optimization problem \citep{ziebart2010modeling, ho2016generative}. In Proposition~\ref{prop:consistency} below we show that the same is true for constrained IRL.
\begin{proposition}\label{prop:consistency}
If Assumption~\ref{ass:realizability} holds, then the rewards optimizing
\begin{equation}\label{eq:cirl_occ_exact}\tag{IRL}
    \min_{\fr\in\mathcal{R}}\max_{\fmu\in\mathcal{F}}\;\;\fr^\top\br{\fmu - \fmuE} - f(\fmu),
\end{equation}
are exactly those rewards in $\mathcal{R}$ for which the expert occupancy measure is optimal in problem \eqref{eq:cmdp_occ}.
\end{proposition}
The proof of Proposition~\ref{prop:consistency} can be found in Appendix~\ref{app:subsec:prop:consistency}. The intuition behind \eqref{eq:cirl_occ_exact} is to seek a reward $\fr\in\mathcal{R}$ for which the suboptimality of the expert is minimized. If Assumption~\ref{ass:realizability} fails -- that is, $\mathcal{R}$ does not contain a reward for which the expert is optimal -- then \eqref{eq:cirl_occ_exact} finds a reward for which the expert is least suboptimal. Motivated by Proposition~\ref{prop:consistency}, we define the (set-valued) IRL solution map $\IRL:\R^{nm}\to 2^{\mathcal{R}}$ via
\begin{equation}
    \IRL(\fmuE)\defeq \argmin_{\fr\in\mathcal{R}}\max_{\fmu\in\mathcal{F}}\;\;\fr^\top\br{\fmu - \fmuE} - f(\fmu).
\end{equation}
Additionally, we let $\IRLO \defeq \mathsf{IRL}_{\R^{nm}, \mathcal{F}}$ for the unrestricted reward class $\mathcal{R} = \R^{nm}$. Analogously, $\mathsf{IRL}_{\mathcal{R}, \mathcal{M}}$ and $ \mathsf{IRL}_{\mathcal{M}}$ are the solution maps for unconstrained IRL. Equipped with the above definitions, we can rewrite Proposition~\ref{prop:consistency} as
\begin{equation}
    \IRL(\fmuE) = \bc{\fr\in\mathcal{R}: \fmuE\in\RL(\fr)}.
\end{equation}
Furthermore, two simple consequences of Proposition~\ref{prop:consistency} are summarized in the following corollary.
\begin{corollary}\label{cor:consequences_consistency}
    If Assumption~\ref{ass:realizability} holds, then \vspace{0.3cm}\\
    \,(a) \; $\IRL(\fmuE) = \IRLO(\fmuE)\cap\mathcal{R}.$ \vspace{0.3cm}\\
    If additionally Assumption~\ref{ass:strict_convexity} holds, then \vspace{0.3cm}\\
    \,(b) \; $(\RL\circ\IRL)(\fmuE) = \bc{\fmuE}.$
\end{corollary}
Here, (a) states that once we know $\IRLO(\fmuE)$ we can recover $\IRL(\fmuE)$ for any realizable expert via intersection with the reward class $\mathcal{R}$, and (b) shows that if the regularization $f$ is strictly convex, we uniquely recover the expert from the learned reward. In contrast, without strict convexity, there may be trivial solutions to the IRL problem that do not provide any insight into the expert's behavior. For example, let $f=0$ and $\bm 0\in\mathcal{R}$. Then, all occupancy measures are optimal for $\bm 0$ and hence we have $(\RL\circ\IRL)(\fmuE) = \mathcal{F}$ for any expert occupancy measure $\fmuE\in\mathcal{F}$. A popular strictly convex regularization in IRL is the negative Shannon entropy \eqref{eq:entropy_reg} leading to the widely used MCE-IRL algorithm \citep{ziebart2010modeling}. Moreover, other regularizations considered in the IRL literature are sparse Tsallis entropy \citep{lee2018maximum} or exponential policy regularization \citep{jeon2020regularized}. 

Next, we will explicitly characterize the set of rewards $\IRL(\fmuE)$ that can be recovered via constrained IRL. To this end, we will first consider the unrestricted reward class $\mathcal{R}=\R^{nm}$ and recover $\IRL(\fmuE)$ via Corollary~\ref{cor:consequences_consistency}(b).

\parSpace\vspace{-0.2cm}
\paragraph{Identifiability}
 Trivially, the optimal occupancy measure in \eqref{eq:cmdp_occ} is invariant to constant shifts of the reward. Furthermore, \citet{ng1999policy} show that if the reward is allowed to depend on the consecutive state $s'\sim \fP(\cdot|s,a)$, the so-called \emph{potential shaping transformations} $\bar{\fr}(s,a,s')\mapsto\bar{\fr}(s,a,s')+\feta(s) - \gamma \feta(s')$ leave the optimal occupancy measure in \eqref{eq:cmdp_occ} invariant. Using the conversion $\fr(s,a)\defeq \E_{s'\sim \fP(\cdot|s,a)} \bar{\fr}(s,a,s')$, potential shaping reduces in our setting to $\fr\mapsto \fr  + \Delta_{\fr}$ with
\begin{alignat}{3}
    &\Delta_{\fr}\in\mathcal{U} &&\defeq \big\{\Delta_{\fr}\in&&\R^{nm}:\; \Delta_{\fr}(s,a) = \feta(s) \\
    &&&&&- \gamma\E_{s'\sim \fP(\cdot|s,a)} \feta(s'), \;\feta\in\R^n \big\} \nonumber\\
    &&& = \spn(\fE&&-\gamma\fP). \nonumber\vspace{-0.1cm}
\end{alignat}
As a consequence, we expect $\IRLO(\fmuE)\supseteq\frE+\mathcal{U}$. A recent result by \citet{cao2021identifiability} and \citet{skalse2022invariance} shows that for standard unconstrained MCE-IRL it holds
\begin{equation}\label{eq:unconstrained-mce-irl-identifiability}
    \mathsf{IRL}_{\mathcal{M}}(\fmuE) = \frE + \mathcal{U} = \beta\log \fpi^{\fmuE} + \mathcal{U},
\end{equation}
with $\beta(\log \fpi^{\fmuE})(s,a) = \beta\log\fpi^{\fmuE}(a|s) = \nabla f(\fmuE)(s,a)$. In other words, for $\mathcal{R} = \R^{nm}$ the expert's reward can be identified up to potential shaping in MCE-IRL. Example~\ref{ex:1} below shows that for MCE-IRL identifiability up to potential shaping is lost when there are active safety constraints\footnote{We say that an inequality constraint is active if it is satisfied with equality.}.
\vspace{0.2cm}
\begin{example}\label{ex:1}
Consider a single state CMDP with $\mathcal{A}=\bc{a_1, a_2}$ and the constraint $\fmu(a_2) \leq 3/4$ i.e. $\fPsi = \myvec{0, 1}^\top$ and $b=3/4$. In this simplified setting it holds $\fmu^{\fpi}=\fpi$, $\mathcal{M} = \Delta_{\mathcal{A}}$, and $\mathcal{U}=\spn(\ones_2)$. Let the expert $\fmuE$ be optimal for $\frE=[0,2]^\top$ and the entropy regularization $f(\fmu)=\E_{a\sim\fmu}\log \fmu(a)$. Then, by solving the optimality conditions we can show that $\fmuE = [1/4, 3/4]^\top$ and  $\nabla f(\fmuE) \approx [-0.39, 0.71]^\top$. However, as illustrated in Figure~\ref{fig:rev:ex1}, we have $\frE\notin \mathsf{IRL}_{\mathcal{M}}(\fmuE) = \nabla f(\fmuE) + \mathcal{U}$. In fact, starting from any $\fr\in\mathsf{IRL}_{\mathcal{M}}(\fmuE)$ we can -- due to the active safety constraint -- increase $\fr(a_2)$ without affecting optimality of $\fmuE$. Hence, the expert is optimal for all $\fr\in\nabla f(\fmuE) + \mathcal{U} + \cone \fPsi$.
\end{example}
\begin{figure}[h!]
     \centering
     \includegraphics[width=0.48\textwidth]{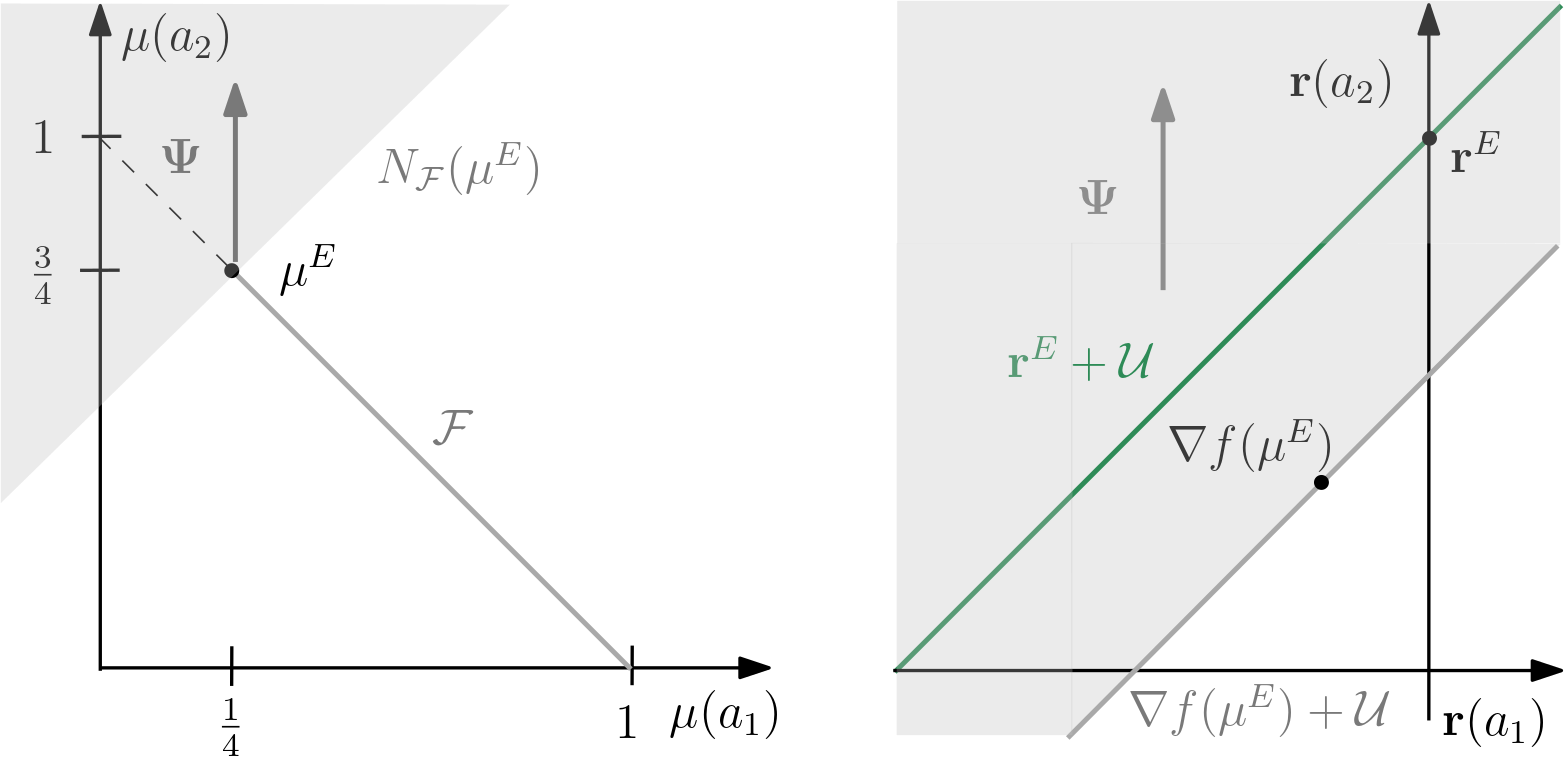}
     \vspace{-0.5cm}
    \begin{flushleft}\hspace{1cm}\footnotesize (a) Feasible set. \hspace{2.4cm}(b) Rewards.\end{flushleft}\vspace{-0.1cm}
     \vspace{-0.2cm}
     \caption{(a) illustrates the feasible set $\mathcal{F}$ and the normal cone $N_{\mathcal{F}}(\fmuE)$. (b) shows $\mathsf{IRL}_{\mathcal{M}}(\fmuE) = \nabla f(\fmuE) + \mathcal{U}$ and $\IRLO(\fmuE)=\nabla f(\fmuE) + \mathcal{U} + \cone \fPsi$.}\vspace{-0.3cm}
     \label{fig:rev:ex1}
\end{figure}
Our main result of this section shows that more generally identifiability up to potential shaping is lost whenever there are active inequality constraints. In particular, under Assumption~\ref{ass:slater} an occupancy measure $\fmu\in\mathcal{F}$ is optimal for some reward if and only if this reward is contained in the Minkowski sum of the subdifferential of $f$ and the normal cone to $\mathcal{F}$ at $\fmu$. Moreover, the normal cone decomposes into the linear subspace of potential shaping transformation $\mathcal{U}$ and additional conic combinations of the gradients of active inequality constraints. The latter become nonzero when $\fmu$ lies on the relative boundary of the feasible set. In this case, we may have $\IRLO(\fmuE)\supset \frE+\mathcal{U}$ and $\frE\notin\mathsf{IRL}_{\mathcal{M}}(\fmuE)$.
\begin{theorem}\label{thm:identifiability}
    Let Assumption~\ref{ass:slater} hold and consider $\fmu\in\mathcal{F}$. Let $\mathcal{I}(\fmu)$ and $\mathcal{J}(\fmu)$ denote the set of indices of active inequality constraints under $\fmu$ i.e. $\fPsi_{i}^\top\fmu=\fb_i$ and $\fmu(s,a) = 0$ if and only if $i\in\mathcal{I}(\fmu)$ and $(s,a)\in\mathcal{J}(\fmu)$. Then,
    \begin{equation}\label{eq:identifiability_general}
        \fmu\in \RL(\fr) \iff \fr\in\partial f(\fmu) + N_{\mathcal{F}}(\fmu),
    \end{equation}
    where $N_{\mathcal{F}}(\fmu) = \mathcal{U} + \mathcal{C}(\fmu) + \mathcal{E}(\fmu)$ with 
    \begin{align}
        \mathcal{C}(\fmu)&\defeq \cone\br{\bc{\fPsi_i}_{i\in \mathcal{I}(\fmu)}},\nonumber\\
        \mathcal{E}(\fmu)&\defeq \cone\br{\bc{-\fe_{s,a}}_{(s,a)\in\mathcal{J}(\fmu)}}.\nonumber
    \end{align}
    Here, $\fe_{s,a}\in\R^{nm}$ denote the standard unit vectors with $\fe_{s,a}(s',a')=1$ if $(s,a)=(s',a')$ and $\fe_{s,a}(s',a')=0$ otherwise.
\end{theorem}
\begin{remark}\label{rem:identifiability}
    Note that in the differentiable case, i.e. if $\partial f(\fmu)=\bc{\nabla f(\fmu)}$, the right-hand-side in \eqref{eq:identifiability_general} reduces to the standard first-order optimality condition
    \begin{equation}
        \nabla h(\fmu)^\top\br{\fmu'-\fmu}\leq 0, \forall \fmu' \in \mathcal{F},
    \end{equation}
    for maximization of $h(\fmu)=\fr^\top \fmu - f(\fmu)$ over the set $\mathcal{F}$. Furthermore, for entropy regularization we have $\partial f(\fmu) = \emptyset$ for $\fmu \in\relbd\mathcal{M}$ (see Corollary~\ref{app:cor:essential_smoothness}), which by condition~\eqref{eq:identifiability_general} ensures that the optimal occupancy measure $\fmu$ lies in the relative interior of $\mathcal{M}$ and hence $\mathcal{E}(\fmu)=\bm 0$.\looseness-1
\end{remark}
The proof of Theorem~\ref{thm:identifiability} rests on the optimality conditions for the CMDP problem~\eqref{eq:cmdp_occ} and is provided in Appendix~\ref{app:subsec:thm:identifiability}. Moreover, we provide an extension to state-action-state rewards in Appendix~\ref{app:subsec:state-action-state-rewards}. In light of the above result and Proposition~\ref{prop:consistency}, the rewards recovered via constrained IRL are characterized as follows:\looseness-1
\begin{corollary}\label{cor:identifiability_irl}
    Let Assumption~\ref{ass:slater} and \ref{ass:realizability} hold. Then, 
    \begin{equation}
        \IRLO(\fmuE) = \partial f(\fmuE) + \mathcal{U} + \mathcal{C}(\fmuE) + \mathcal{E}(\fmuE).
    \end{equation}
\end{corollary}
\vspace{-0.2cm}
Corollary~\ref{cor:identifiability_irl} shows that whenever $\fmuE\in\relbd\mathcal{F}$, then $\mathcal{C}(\fmuE)$ or $\mathcal{E}(\fmuE)$ is nonzero. From this, we observe two points. First, in the case of active safety constraint, we lose identifiability up to potential shaping i.e. \eqref{eq:unconstrained-mce-irl-identifiability} no longer holds. Second, in the case in which the feasible set is $\mathcal{M}$ (no safety constraints), the expert's reward is identifiable up to potential shaping if $\fmuE$ lies in $\relint\mathcal{M}$. While the entropy regularization ensures that any optimal occupancy measure lies in $\relint\mathcal{M}$, the following example shows that this is not the case for a $2$-norm regularization.

\begin{example}
    Consider the same MDP as in Example~\ref{ex:1}, but without constraint. Let again $\frE=[0,2]^\top$, and let $\fmuE_1$ be optimal for $\frE$ with regularization $f_1(\fmu)=\E_{a\sim\fmu}\log \fmu(a)$, and $\fmuE_2$ for $f_2(\fmu)=\norm{\fmu}_2^2/2$. Then, it can be shown that $\fmuE_1 = [0.12,0.88]^\top, \fmuE_2 = [0, 1]^\top$ and $\nabla f_1(\fmuE_1) \approx [-1.13, 0.87]^\top, \nabla f_2(\fmuE_2) = [0, 1]^\top$. While for $f_1$ it holds $\mathsf{IRL}_{1,\mathcal{M}}(\fmuE_1)=\frE+\mathcal{U}$, for $f_2$ the active non-negativity constraint $\fmuE_2(a_1)\geq 0$ allows us to decrease $\fr(a_1)$ without affecting optimality of $\fmuE_2$ (see Figure~\ref{fig:rev:ex2}). Moreover, according to Theorem 4.4 we have $\mathsf{IRL}_{2,\mathcal{M}}(\fmuE_2)=\nabla f(\fmuE_2)+\mathcal{U} + \cone(-\fe_{a_1})$.
\end{example}
\begin{figure}[h!]
     \centering
     \includegraphics[width=0.48\textwidth]{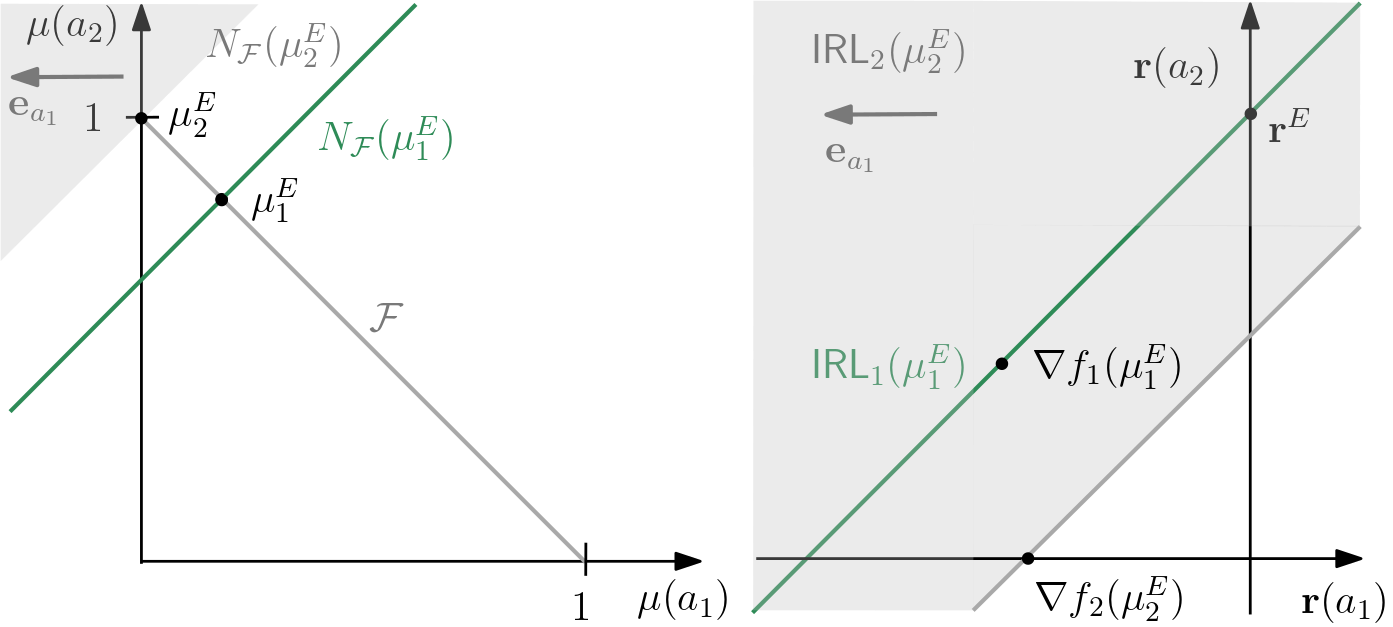}
     \vspace{-0.5cm}
    \begin{flushleft}\hspace{1cm}\footnotesize{(a) Feasible set.} \hspace{2.4cm}\footnotesize{(b) Rewards.}\end{flushleft}\vspace{-0.1cm}
     \vspace{-0.2cm}
     \caption{(a) illustrates the feasible set $\mathcal{F}$ and the normal cones $N_{\mathcal{F}}(\fmuE_1), N_{\mathcal{F}}(\fmuE_2)$. (b) shows $\mathsf{IRL}_{1,\mathcal{M}}(\fmuE_1)$ under $f_1$ (in green), and $\mathsf{IRL}_{2,\mathcal{M}}(\fmuE_2)$ under $f_2$ (in gray).}
     \label{fig:rev:ex2}\vspace{-0.2cm}
\end{figure}
More generally, the optimal occupancy measure is guaranteed to lie in $\relint\mathcal{M}$ if the gradient of the regularization becomes unbounded when approaching $\relbd\mathcal{M}$. This is formalized in Assumption~\ref{ass:essential_smoothness} and Corollary~\ref{cor:essential_smoothness_consequence} below.
\begin{assumption}\label{ass:essential_smoothness}
    Let $f:\mathcal{X}\to \R$ be such that:\vspace{-0.3cm}
\begin{enumerate}[(a)]
    \item $f$ is differentiable throughout $\interior\mathcal{X}$,\vspace{-0.2cm}
    \item $\lim_{k\to \infty} \norm{\nabla f(\fmu_k)}= \infty$ if $\br{\fmu_k}_{k\in\mathbb{N}}$ is a sequence in $\interior\mathcal{X}$ converging to a point $\fmu\in \relbd \mathcal{M}$.
\end{enumerate}
\end{assumption}
\begin{corollary}\label{cor:essential_smoothness_consequence}
    Let Assumptions~\ref{ass:slater}, \ref{ass:realizability}, \ref{ass:essential_smoothness} hold. Then, we have $\RL(\fr) \subset \relint\mathcal{M}$ for any $\fr\in\R^{nm}$ and
    \begin{equation}
        \IRLO(\fmuE) = \nabla f(\fmuE) + \mathcal{U} + \mathcal{C}(\fmuE).
    \end{equation}
\end{corollary}
\vspace{-0.2cm}
The proof of Corollary~\ref{cor:essential_smoothness_consequence} is provided in Appendix~\ref{app:subsec:cor:essential_smoothness_consequence}. Assumption~\ref{ass:essential_smoothness} is satisfied for the entropy regularization \eqref{eq:entropy_reg} or relative entropy regularization (see Corollary~\ref{app:cor:essential_smoothness}). However, certainly it is not satisfied for many other choices of regularization such as $f=0$, the sparse Tsallis entropy \citep{lee2018sparse}, or the $2$-norm regularization.
\vspace{-0.2cm}
\parSpace
\paragraph{Generalizability}
As for our initial goal of learning a reward generalizing to new transition laws and constraints, the result of Corollary~\ref{cor:identifiability_irl} is problematic, since the set $\IRLO(\fmuE)$ depends on both the transition law and the constraints. To make this dependency explicit, we will throughout this section denote $\RL^{\fP, \fb}, \IRLO^{\fP, \fb}$ for the solution maps corresponding to the transition law $\fP$ and constraint threshold $\fb$. Additionally, we let $\mathfrak{P}\subset \R^{nm\times n}$ be the set of all transition laws. Moreover, we introduce the following notion of generalizability.
\begin{definition}[Generalizability]\label{def:generalizability}
    Fix some transition law and constraint threshold $(\fP_0, \fb_0)$. We say that IRL generalizes to $\mathcal{P}\subseteq \mathfrak{P}$ and $\mathcal{B}\subseteq \R^k$, if
    \begin{equation}\label{eq:generalizability}
        \RL^{\fP, \fb}(\fr) = \RL^{\fP, \fb}(\fr'),\;\forall\fP\in\mathcal{P}, \forall\fb\in\mathcal{B},
    \end{equation}
    for any pair of rewards $\fr, \fr' \in\IRL^{\fP_0, \fb_0}(\fmuE)$.
\end{definition}
\vspace{-0.2cm}
Definition~\ref{def:generalizability} requires all rewards recovered via IRL to yield the same optimal occupancy measures for all $\fP\in\mathcal{P}$ and $\fb\in\mathcal{B}$. The subsequent result shows that, under Assumption~\ref{ass:essential_smoothness}, generalization to a neighborhood of new transition laws and arbitrary constraint thresholds is possible if and only if the expert's reward is recovered up to a constant. 
\begin{theorem}\label{thm:intersection}
    Let Assumptions~\ref{ass:slater}, \ref{ass:strict_convexity}, \ref{ass:realizability}, \ref{ass:essential_smoothness} hold for $(\fP_0,\fb_0)$ and let $\fmuE\in\RL^{\fP_0,\fb_0}(\frE)$ for some $\frE\in\mathcal{R}$. Consider an arbitrary neighborhood $\mathcal{O}_{\fP_0}\subseteq \R^{nm\times n}$ of $\fP_0$. Then, IRL generalizes to $\mathcal{P} = \mathcal{O}_{\fP_0}\cap\mathfrak{P}$ and $\mathcal{B}=\R^k$ if and only if
    \begin{equation}
        \IRL^{\fP_0, \fb_0}(\fmuE) \subseteq \frE + \spn(\ones_{nm}).
    \end{equation}
\end{theorem}
\vspace{-0.2cm}
Note that since $\mathcal{B}=\R^k$, Theorem~\ref{thm:intersection} considers generalizability to all possible constraint thresholds -- in particular to the unconstrained setting.\footnote{For the unrestricted reward class $\mathcal{R}=\R^{nm}$, we expect the same result to hold even if $\mathcal{B}$ is only a neighborhood of constraint thresholds. However, a rigorous proof would require continuity of $\fP\mapsto\RL^{\fP}(\fr)$, which we leave open to future work.} The main idea of the proof of Theorem~\ref{thm:intersection} is to show that in any neighborhood of $\fP_0$ we can find $\fP_1, \fP_2\in\mathfrak{P}$ such that only the rewards in $\frE + \spn(\ones_{nm})$ generalize to both $\fP_1$ and $\fP_2$. To the best of our knowledge, this is a novel result -- even in the context of unconstrained IRL. For more details about the proof and a brief discussion about the connection to recent results on identifiability from multiple experts \citep{cao2021identifiability, rolland2022identifiability}, we refer to Appendix~\ref{app:subsec:thm:intersection}. 

In light of Theorem~\ref{thm:intersection}, the following corollary shows that for the unrestricted reward class, IRL is not generalizing to a neighborhood of new transition laws and arbitrary constraints.
\begin{corollary}\label{cor:non-gen}
    Let Assumption~\ref{ass:slater} and \ref{ass:realizability} hold and let $n>1$ and $\mathcal{R}=\R^{nm}$. Moreover, let $\mathcal{B}$ and $\mathcal{P}$ be defined as in Theorem~\ref{thm:intersection}. Then, IRL is not generalizing to $\mathcal{P}$ and $\mathcal{B}$.
\end{corollary}
Corollary~\ref{cor:non-gen} is a consequence of $\dim \mathcal{U} = n$, which implies that $\IRLO^{\fP_0, \fb_0}(\fmuE) \supset \frE + \spn(\ones_{nm})$. Below, we show that the above problem can be resolved by a suitable restriction of the reward class. In particular, if the rank condition in Proposition~\ref{prop:linear_reward_identifiability} holds, then the reward class intersects the space spanned by potential shaping transformations and the safety cost only at the origin, which ensures that the expert's reward can be identified exactly. 
\begin{proposition}\label{prop:linear_reward_identifiability}
     Let Assumptions~\ref{ass:slater}, \ref{ass:strict_convexity}, \ref{ass:realizability}, \ref{ass:essential_smoothness} hold. Moreover, let $\fmuE\in\RL^{\fP_0, \fb_0}(\frE)$ for some $\frE\in\mathcal{R}$ and
    \begin{equation}
        \mathcal{R} \subseteq \bc{\fr_{\fw} = \fPhi \fw: \fPhi \in\R^{mn\times d}, \fw\in\R^d}.
    \end{equation}
    Then, if for $\fXi\defeq\myvec{\fE - \gamma \fP_0, \fPsi}$ it holds that
    \begin{equation}\label{eq:cirl_identifiability_condition}
        \rank \myvec{\fPhi, \fXi} - \br{\rank \fPhi + \rank \fXi} = 0,
    \end{equation}
    then we have $\IRL^{\fP_0, \fb_0}(\fmuE) = \bc{\frE}$.
\end{proposition}
The proof of Proposition~\ref{prop:linear_reward_identifiability} is provided in Appendix~\ref{app:subsec:cor:linear_reward_identifiability}. Observe that for a known transition law, condition~\eqref{eq:cirl_identifiability_condition} can easily be verified. 
\vspace{-0.3cm}
\paragraph{Ignoring the Constraints}
As for strictly convex regularization, we have the equivalence $\RL(\fr) = \RLO(\fr-\fPsi\fxi^*)$ (Proposition~\ref{prop:strong_duality}), we may ignore safety constraints in IRL and recover the modified reward $\fr - \fPsi\fxi^*$ via unconstrained IRL. However, this comes with two caveats. First, the reward class needs to be sufficiently expressive to implicitly account for the safety constraints. Second, as illustrated in Example~\ref{ex:1}, $ \mathsf{IRL}_{\mathcal{M}}(\fmuE)$ may not contain the expert's reward in case of active safety constraints and hence fail to generalize to the unconstrained setting.\looseness-1
\vspace{-0.2cm}
\section{Finite Sample Setting}\label{sec:finite_sample}
\vspace{-0.1cm}
\paragraph{Practical Inverse Reinforcement Learning}
In practice, we only have access to a finite data set of demonstrations $\mathcal{D} = \bc{(s^i_t, a^i_t)_{t=0}^T}_{i=1}^N$. In this case,  the expert occupancy measure $\fmuE$ can be estimated via \citep{abbeel2004apprenticeship}
\begin{equation}
    \fmuEhat(s,a)\defeq \dfrac{(1-\gamma)}{N}\sum_{i=1}^N\sum_{t=0}^T \gamma^t \mathbbm{1}(s_t^i=s, a_t^i=a),
\end{equation}
where $\mathbbm{1}$ is an indicator function. Swapping the order of minimization and maximization using Sion's min-max theorem \citep{sion1958general}, we can interpret the resulting min-max problem as the dual of an occupancy measure matching problem \citep{syed2007game, syed2008apprenticeship, ho2016generative}\vspace{-0.1cm}
\begin{align}\label{eq:cirl_occ_emp}
    &\min_{\fr\in\mathcal{R}}\max_{\fmu\in\mathcal{F}}  \;\;\fr^\top\br{\fmu - \fmuEhat} - f(\fmu)\\
    &=-\max_{\fr\in\mathcal{R}}\min_{\fmu\in\mathcal{F}}  \;\; \bs{-\fr^\top\br{\fmu - \fmuEhat} + f(\fmu)}\nonumber\\
    &=-\min_{\fmu\in\mathcal{F}} \bs{\delta_{\mathcal{R}}(\fmu, \fmuEhat) + f(\fmu)}.\nonumber
\end{align}
Here, $\delta_{\mathcal{R}}(\fmu, \fmuEhat) \defeq \max_{\fr\in\mathcal{R}} \fr^\top\br{\fmuEhat-\fmu}$ is an integral probability metric \citep{muller1997integral} measuring the distance from $\fmu$ to the empirical expert occupancy measure. For different choices of $\mathcal{R}$ different distance measures arise. The choice $\mathcal{R}= \R^{nm}$ yields a characteristic function with $\delta_{\mathcal{R}}(\fmu, \fmuEhat)=0$ if $\fmu=\fmuEhat$, and $\delta_{\mathcal{R}}(\fmu, \fmuEhat)=\infty$ otherwise \citep{boyd2004convex}. For the bounded linear feature classes $\mathcal{R}^{\norm{\cdot}}\defeq\bc{\fr_{\fw}=\fPhi \fw: \fPhi\in\R^{nm\times d}, \norm{\fw}\leq 1}$ we get $\delta_{\mathcal{R}}(\fmu, \fmuEhat) = \norm{\fPhi^\top (\fmu-\fmuEhat)}_*$, where $\norm{\cdot}_*$ denotes the dual norm to $\norm{\cdot}$. Thus, for $\norm{\cdot}_{2}$ we recover feature expectation matching in the 2-norm \citep{abbeel2004apprenticeship} and for $\norm{\cdot}_{1}$ in the $\infty$-norm \citep{syed2008apprenticeship}. Other choices of $\mathcal{R}$ lead to other distance measures such as the Wasserstein-1 distance or the maximum mean discrepancy (for an overview see \citep{xiao2019wasserstein, sun2019provably, swamy2021moments}). In the following, we focus on the choice\looseness-1
\begin{align}\label{eq:reward_classes}
    \mathcal{R}^{\norm{\cdot}_1} \defeq\bc{\fr_{\fw}=\fPhi \fw: \fPhi\in\R^{nm\times d}, \norm{\fw}_1\leq 1}.
\end{align}
We will see that bounding the reward class as above enables us to derive a bound for the sample complexity of IRL.
\parSpace\vspace{-0.2cm}
\paragraph{Sample Complexity}
The subsequent result shows that if the reward class is bounded, we can bound the suboptimality of solutions obtained by the finite sample problem \eqref{eq:cirl_occ_emp} with respect to the idealized problem~\eqref{eq:cirl_occ_exact}.
\begin{theorem}\label{thm:sample_complexity}
     Let Assumption~\ref{ass:realizability} hold and $\fmuE\in\RL(\frE)$ for some $\frE\in\mathcal{R}\defeq\mathcal{R}^{\norm{\cdot}_1}$. Let $\fmuhat \in \RL\circ\IRL(\fmuEhat)$ and $R\defeq\max_{s,a}\norm{\fPhi(s,a)}_{\infty}$. Choosing \vspace{-0.1cm}
     \begin{equation}
         N = \ceil{\dfrac{32R^2}{\varepsilon^2}\log\br{\dfrac{2d}{\delta}}} \text{ and } T=\ceil{\log\br{\dfrac{\varepsilon}{8R}}/\log(\gamma)}.
     \end{equation}\vspace{-0.1cm}
     It holds with probability at least $1-\delta$
    \begin{align}\label{eq:value_diff}
        &J(\fmuE, \frE) - J(\fmuhat, \frE) \leq \varepsilon\vspace{-0.1cm},
    \end{align}
    where $J(\fmu,\fr)\defeq \fr^\top\fmu - f(\fmu)$. Moreover, if\vspace{-0.2cm}
    \begin{enumerate}[(a)]
        \item $f$ is $L$-strongly convex with respect to the norm $\norm{\cdot}$, it holds with probability at least $1-\delta$\vspace{-0.1cm}
        \begin{equation}\label{eq:occ_diff}
            \norm{\fmuhat - \fmuE} \leq \sqrt{\dfrac{2\varepsilon}{L}}.
        \end{equation}
        \vspace{-0.2cm}
        \item $f(\fmu)=-\beta \,\E_{(s,a)\sim\fmu}\bs{H\left(\fpi^{\fmu}(\cdot|s)\right)}$ with $\beta>0$, it holds with probability at least $1-\delta$\vspace{-0.1cm}
        \begin{equation}\label{eq:policy_diff}
         \E_{(s,a)\sim\fmuE}\bs{\norm{\fpi^{\fmuhat}(\cdot|s)-\fpiE(\cdot|s)}_1}\leq \sqrt{\dfrac{2\varepsilon}{\beta}}.
        \end{equation}
    \end{enumerate}
\end{theorem}
\begin{figure*}[h!]
    \centering
    \includegraphics[width=1.\textwidth]{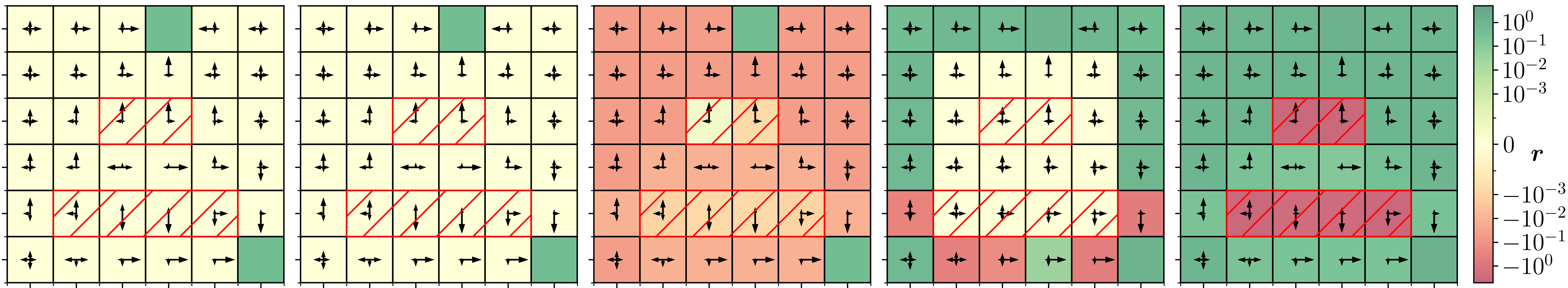}
    \vspace{-0.5cm}
    \begin{flushleft}\hspace{0.9cm}(\textit{a}) Expert \hspace{1.5cm}(\textit{b}) $\mathsf{IRL}_{\mathcal{R}_1, \mathcal{F}}$\hspace{1.5cm}(\textit{c}) $\mathsf{IRL}_{\mathcal{R}_2, \mathcal{F}}$\hspace{1.5cm}(\textit{d}) $\mathsf{IRL}_{\mathcal{R}_1, \mathcal{M}}$\hspace{1.5cm}(\textit{e}) $\mathsf{IRL}_{\mathcal{R}_2, \mathcal{M}}$\end{flushleft}\vspace{-0.1cm}
    \caption{Comparing rewards and policies recovered via constrained and unconstrained IRL for the two reward classes $\mathcal{R}_1$ and $\mathcal{R}_2$. The color indicates the reward, arrows the policies, and the two red hatched rectangles the constrained states. (\textit{a}) depicts $\frE$ and $\fpiE$, (\textit{b})-(\textit{c}) the rewards and policies learned from constrained IRL, and (\textit{d})-(\textit{e}) the rewards and policies learned from unconstrained IRL.\label{fig:rewards}\looseness-1}
\end{figure*}
The first result of Theorem~\ref{thm:sample_complexity} shows that by collecting enough expert trajectories with large enough time horizon, constrained IRL recovers with high probability an occupancy measure which is only $\varepsilon$-suboptimal under the expert's reward. A similar result has been shown by \citet{syed2007game} for unconstrained unregularized IRL. Second, we show that under strong convexity the recovered occupancy measure is close to the expert's one. Moreover, although the regularization is not strongly convex in MCE-IRL\footnote{While the entropy itself is $1$-strongly convex in the $1$-norm, the resulting regularization $f$ is in general not satisfying this property. However, it is still strictly convex as shown in Proposition~\ref{prop:convexity}.}, we are still recovering a policy that is close to the expert's policy -- at least under the support of the expert occupancy measure. To the best of our knowledge, closeness to the expert occupancy measure (or policy) is novel and only holds in the regularized setting. 

\headSpaceBefore\vspace{-0.2cm}
\section{Experimental Results}\label{sec:experiments}
\paragraph{Setup}
To validate our results, we consider a gridworld environment \citep{sutton2018reinforcement} with $36$ states (the grid cells) and 4 actions (up, down, left, right).\footnote{The code to all our experiments is available at: \\\url{https://github.com/andrschl/cirl}} The agent has a 90\% chance of reaching the desired location when taking an action and a 10\% chance of ending up in a random neighboring grid cell. We choose the entropy regularization $f(\fmu) = -\E_{(s,a)\sim\fmu}H(\fpi^{\fmu})$, and consider rewards that are only state-dependent, namely, two linear reward classes \vspace{-0.1cm}
\begin{equation}
        \mathcal{R}_1 \defeq \bc{\fPhi_1 \fw: \fw\in\R^{20}} \text{ and } \mathcal{R}_2 \defeq \bc{\fPhi_2 \fw: \fw\in\R^{36}}.
\end{equation}\vspace{-0.0cm}
$\mathcal{R}_1$ has a single reward feature for every state on the boundary, and $\mathcal{R}_2$ has reward features for all states. That is, ${\fPhi_1} = [\fE_{i_1}, \hdots, \fE_{i_{20}}]$ and ${\fPhi_2} = \fE$ and, where $i_1,\hdots,i_{20}$ are the indices corresponding to states on the boundary of the gridworld and $\fE$ is the matrix as defined in \eqref{eq:bellman_flow}. The rank condition of Corollary~\ref{cor:identifiability_irl} is satisfied for the smaller reward class $\mathcal{R}_1$, but not for $\mathcal{R}_2$. The expert's reward $\frE$ is depicted in Figure~\ref{fig:rewards}(a). It is zero everywhere except for the two green grid cells where $\frE(s,\cdot)=0.5$. Furthermore, there are two safety constraints indicated by the red-hatched rectangles. The two rectangular constraints are enforced separately via $\fPsi_1, \fPsi_2$ which are one on the constrained cells and zero everywhere else. The constraint threshold is $\fb_0 = 0.02\cdot\ones_2$, where feasibility is checked via the LP solver \verb|linprog| provided by \citep{2020SciPy-NMeth}.
\vspace{-0.2cm}
\paragraph{Algorithm}
As an algorithm for the min-max problem \eqref{eq:cirl_occ_exact} we use a primal-dual gradient-descent-ascent method \citep{daskalakis2018limit} in the policy space (instead of occupancy measure space). In particular, we update the reward parameters and dual variables for the constraints via a (projected) gradient descent step, and the policy via an entropy-regularized natural policy gradient \citep{cen2022fast} step. The algorithm is provided in Appendix~\ref{app:sec:algorithm}.
\vspace{-0.2cm}
\paragraph{Generalizability}
Learning from the \emph{true expert occupancy measure}, we compare the rewards recovered for constrained vs. unconstrained IRL and $\mathcal{R}_1$ vs. $\mathcal{R}_2$. Figure~\ref{fig:rewards} illustrates the rewards and policies recovered during IRL. Furthermore, Table~\ref{tab:irl_rewards} above summarizes occupancy measure errors and suboptimality for generalization to the same constrained setting as in training ($\fb_0$) and an unconstrained test setting (with $\fb_1\gg\fb_0$). 
\begin{table}[t]
\caption{Comparing generalization of the learned rewards for different constraint thresholds. \emph{Train} indicates the constrained setting with threshold $\fb_0$ (as used in training), and \emph{test} the generalization to the unconstrained setting (by setting $\fb_1$ large).}\vspace{-0.3cm}
\label{tab:irl_rewards}
\vskip 0.15in
\begin{center}
\begin{footnotesize}
\begin{sc}
\begin{tabular}{l|cc|cc}
\toprule
Method & \multicolumn{2}{l}{Train ($\fb_0$)} & \multicolumn{2}{l}{Test ($\fb_1\gg\fb_0$)}\\
&$\Delta \fmu$ & $\Delta J$ & $\Delta \fmu$ & $\Delta J$ \\
\midrule
$\mathsf{IRL}_{\mathcal{R}_1, \mathcal{F}}$&9.6e-9&5.3e-15&1.3e-7&9.2e-14\\
$\mathsf{IRL}_{\mathcal{R}_2, \mathcal{F}}$&1.7e-6&2.1e-10&2.1e-2& 1.7e-3 \\
$\mathsf{IRL}_{\mathcal{R}_1, \mathcal{M}}$&9.5e-2&1.0e-1&2.4e-1 & 2.9e-1 \\
$\mathsf{IRL}_{\mathcal{R}_2, \mathcal{M}}$&9.3e-3& 1.0e-2&2.8e-1& 5.9e-1 \\
\bottomrule
\end{tabular}
\end{sc}
\end{footnotesize}
\end{center}
\vskip -0.2in
\end{table}
Here, the occupancy measure error and suboptimality are defined via $\Delta\fmu=||\fmu^{\text{E}, \fb} - \fmu^{\fb}||_1$ and $\Delta J = J(\fmu^{\text{E}, \fb}, \frE)-J(\fmu^{\fb}, \frE)$, with $\fmu^{\text{E}, \fb}\in\RL^{\fb}(\frE)$ and $\fmu^{\fb}\in\RL^{\fb}(\frhat)$, where $\frhat$ indicates the reward recovered via IRL and $\fb\in\bc{\fb_1,\fb_2}$.
\begin{figure}[h!]
	\centering
	\includegraphics[width=0.5\textwidth]{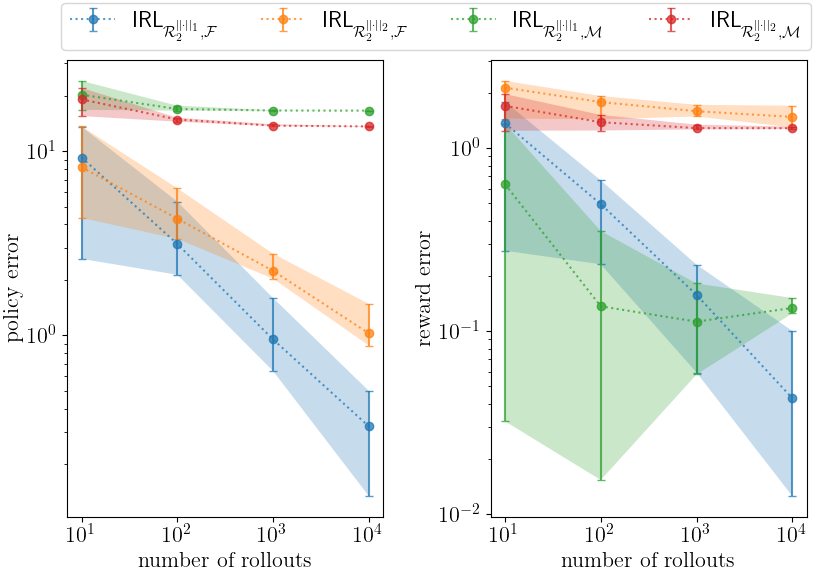}
	\caption{Compare constrained and unconstrained IRL for the two reward classes $\mathcal{R}_2^{\norm{\cdot}_1}$ and $\mathcal{R}_2^{\norm{\cdot}_2}$ and learning from different amount of expert data. The circles indicate the median, and the shaded areas the 0.1 and 0.9 quantiles for 10 independent realizations of the expert data.}
 \label{fig:finite_sample}\vspace{-0.2cm}
\end{figure}
As depicted in Figures~\ref{fig:rewards}(b)-(c) the reward is almost perfectly identified for constrained IRL with the reward class $\mathcal{R}_1$, whereas there is a small mismatch to the expert's reward for $\mathcal{R}_2$. This is in line with our identifiability result in Corollary~\ref{cor:identifiability_irl}, since only $\mathcal{R}_1$ satisfies the rank condition \eqref{eq:cirl_identifiability_condition}. In contrast, the rewards recovered from unconstrained IRL, as shown in Figures~\ref{fig:rewards}(d)-(e), substantially deviate from the expert's reward, since they need to implicitly account for the safety constraints. Table~\ref{tab:irl_rewards} shows that constrained IRL with $\mathcal{R}_1$ clearly outperforms the other methods -- especially in terms of generalization to the unconstrained setting.
\vspace{-0.2cm}
\paragraph{Learning from Expert Data}
To verify the sample complexity result of Theorem~\ref{thm:sample_complexity}, we compare constrained and unconstrained IRL for the two reward classes $\mathcal{R}_2^{\norm{\cdot}_1}$ and \footnote{Since in our experiments the recovered reward was always located on the boundary, we choose the bound $1/\sqrt{2}$ here.}
\begin{equation}
    \mathcal{R}_2^{\norm{\cdot}_2} \defeq\bc{\fr_{\fw}=\fPhi \fw: \fPhi\in\R^{nm\times d}, \norm{\fw}_2\leq 1/\sqrt{2}}.\nonumber
\end{equation}
To this end, we solve the min-max problem \eqref{eq:cirl_occ_emp} for the above reward classes and the feasible sets $\mathcal{F}$ and $\mathcal{M}$. Figure~\ref{fig:finite_sample} shows the policy and reward errors for 10 independent realizations of the expert demonstrations containing $N\in\bc{10,100,1000,10000}$ trajectories of length $T=10000$. 
Here, the policy error is $\norm{\fpiE-\fpihat}_1$, where $\fpihat$ is the policy recovered via IRL, and the reward error is defined as the distance of the recovered reward $\frhat$ to the line $\frE + \spn(\ones_{nm})$. As predicted by Theorem~\ref{thm:sample_complexity}, the policy error is converging towards zero for an increasing number of trajectories. Moreover, the policy error is quite large for unconstrained IRL, which we expect to be due to non-realizability of the expert in this setting (as we need to implicitly account for the constraints). On the other hand, the recovered rewards are -- for both constrained and unconstrained IRL -- much closer to the expert's reward for the reward class $\mathcal{R}_2^{\norm{\cdot}_1}$. We expect the reason for this to be the sparsity induced by the projection onto the 1-norm ball \citep{tibshirani1996regression}, which helps to recover the expert's true reward in this setting. This showcases the importance of the choice of norm when using a bounded linear reward class.
\vspace{-0.2cm}
\headSpaceBefore 
\section{Limitations and Future Work}
For ease of exposition, we limit the scope of this paper to discrete state and action spaces. However, an interesting direction for future research would be to extend our results to the continuous setting, where the CMDP problem can be formulated as an infinite-dimensional convex optimization problem involving the occupancy measure \citep{altman1999constrained}. Furthermore, our results are based on optimal solutions and rewards, but in practical settings, we hardly ever obtain optimal solutions and approximately optimal solutions are the norm. Hence, examining identifiability and generalizability in an approximate setting would be valuable for practical applications and may reveal valuable insights on how to choose the regularization $f$. Finally, Proposition~\ref{prop:linear_reward_identifiability} provides a sufficient condition for generalizability, but checking the rank condition \eqref{eq:cirl_identifiability_condition} requires knowledge of the transition law and the constraints. To alleviate this, it may be helpful to learn a reward from multiple experts with different transition laws and constraints.
\vspace{-0.2cm}
\section{Summary}
In this paper, we present a constrained IRL framework for CMDPs with arbitrary convex regularizations of the occupancy measure. From a convex-analytic viewpoint, we address identifiability and generalizability to new transition laws and constraints. Our results indicate that identifiability of rewards up to potential shaping is contingent on the use of entropy regularizations and that generalizability to new transition laws and constraints is only possible when the expert's reward is identified up to a constant. Based on these insights, we provide a sufficient condition for identifiability and generalizability. Furthermore, we show a novel result on the number of expert trajectories required to recover a reward whose optimal policy is close to the expert's policy. Lastly, we showcase the applicability of our results in a gridworld experiment.
\headSpaceAfter\vspace{-0.2cm}
\section*{Acknowledgements} Andreas Schlaginhaufen is funded by a PhD fellowship from the Swiss Data Science Center.
\newpage
\bibliography{refs}
\bibliographystyle{icml2023}

\newpage
\appendix
\onecolumn
\section{Notation}\label{app:sec:notation}
In the following, we briefly recall a few basic definitions from convex analysis \citep{rockafellar1970convex, boyd2004convex}. To this end, we denote $B(\fx,r)\defeq \bc{\fx\in \R^l:\norm{\fx}_2 <r}$ for an open ball of radius $r$ and center $\fx$.
\begin{definition}[Interior]
The interior of a set $\mathcal{X}\subseteq\R^l$ is defined as
\begin{equation}
    \interior\mathcal{X} \defeq \bc{\fx\in\mathcal{X}: B(\fx,r)\subseteq \mathcal{X} \text{ for some } r>0}.
\end{equation}
\end{definition}
\begin{definition}[Affine hull]
The affine hull of a set $\mathcal{X}\subseteq\R^l$ is defined as
\begin{equation}
    \mathbf{aff}\mathcal{X} \defeq \bc{\theta_1 \fx_1 + \hdots + \theta_k \fx_k:\fx_1,\hdots,\fx_k\in\mathcal{X}, \theta_1+ \hdots + \theta_k=1}.
\end{equation}
\end{definition}
\begin{definition}[Relative interior]
The relative interior of a set $\mathcal{X}\subseteq\R^l$ is defined as
\begin{equation}
    \relint\mathcal{X} \defeq \bc{\fx\in\mathcal{X}: B(\fx,r)\cap \mathbf{aff}\mathcal{X}\subseteq \mathcal{X} \text{ for some } r>0}.
\end{equation}
\end{definition}
\begin{definition}[Relative boundary]
The relative boundary of a closed set $\mathcal{X}\subseteq\R^l$ is defined as
\begin{equation}
    \relbd\mathcal{X} \defeq \mathcal{X} \setminus \relint \mathcal{X}.
\end{equation}
\end{definition}
\begin{definition}[Subdifferential] 
A subgradient of a convex function $f:\mathcal{X}\to \R$ with $\mathcal{X}\subseteq\R^l$ at some point $\fx\in\mathcal{X}$ is a vector $\bm{g}\in\R^{l}$ such that $f(\Tilde{\fx})\geq f(\fx) + \bm{g}^\top\br{\Tilde{\fx} - \fx}$ for all $\tilde{\fx}\in\mathcal{X}$. The subdifferential $\partial f (\fx)$ at $\fx\in\mathcal{X}$ is the set of all subgradients at $\fx$.\looseness-1
\end{definition}
\begin{definition}[Normal cone] 
The normal cone $N_{\mathcal{X}}(\fx)$ of a convex set $\mathcal{X}\subseteq\R^l$ at some point $\fx\in\mathcal{X}$ is the set of all $\bm{h}\in\R^{l}$ such that $\bm{h}^\top\br{\Tilde{\fx} - \fx}\leq 0$ for all $\tilde{\fx}\in\mathcal{X}$.
\end{definition}

\section{Proofs and Comments for Section~\ref{sec:background} and \ref{sec:convex_viewpoint}}\label{app:sec:strong_duality}
\subsection{Entropy Regularization}\label{app:subsec:entropy_regularization}
In their work on regularized MDPs \citet{geist2019theory} consider a family of regularized MDPs with the objective
\begin{equation}
    \max_{\fpi\in\Pi} J(\fpi, \fr),
\end{equation}
where $J(\fpi, \fr) \defeq \E_{(s,a)\sim\fmu^{\fpi}} \bs{\fr(s,a) - \Omega(\fpi(\cdot|s))}$ and $\Omega:\Delta_{\mathcal{A}}\to\R$ is strongly convex. Defining the optimal value and q-value function
\begin{align}
    \f{v}^*(s) &\defeq \max_{\fpi\in\Pi} \E_{\fpi}\bs{\sum_{t=0}^{\infty} \gamma^t\bs{ \fr(s_t, a_t) - \Omega(\fpi(\cdot|s_t))}\bigg|s_0 = s}\\
    \f{q}^*(s,a) &\defeq \fr(s,a) + \gamma \E_{s'\sim \fP(\cdot|s,a)}\f{v}^*(s'),
\end{align}
the optimal policy can be shown to be $\fpi^*(\cdot|s)=\nabla\Omega^*(\f{q}^*(s,\cdot))$, where $\nabla\Omega^*$ is the gradient of the convex conjugate $\Omega^*(\f{q}^*(s,\cdot))\defeq \max_{\fd\in\Delta_{\mathcal{A}}}{\f{q}^*(s,\cdot)}^\top\fd - \Omega(\fd)$. For the entropy regularization $\Omega(\fd) = -\beta H(\fd)$ with $\beta>0$, the optimal policy can be shown to have the soft-max form
\begin{equation}\label{app:eq:soft_opt_policy}
    \fpi^*(a|s) = \dfrac{\exp\br{\f{q}^*(s,a)/\beta}}{\sum_{a'}\exp\br{\f{q}^*(s,a')/\beta}}.
\end{equation}
Accordingly, entropy regularization forces the optimal policy to always assign a non-zero probability to each action regularizing the optimal policy towards the uniform distribution. Similar to unregularized MDPs, the optimal policy in entropy regularized MDPs can be computed via value or policy iteration \citep{ziebart2010modeling, haarnoja2018soft}. 

In the occupancy measure, entropy regularization in the policy takes the form
\begin{equation}\label{app:eq:entropy_regularization}
    f(\fmu) = -\beta \,\E_{(s,a)\sim\fmu}\bs{H\left(\fpi^{\fmu}(\cdot|s)\right)} = \beta \sum_{s,a} \fmu(s,a) \log\br{\dfrac{\fmu(s,a)}{\sum_{a'}\fmu(s,a')}}.
\end{equation}
In order to incorporate prior knowledge about the expert's policy, we may also consider the relative entropy regularization 
\begin{equation}\label{app:eq:relative_entropy_regularization}
    f(\fmu) = \beta \,\E_{(s,a)\sim\fmu}\bs{\DKL\left(\fpi^{\fmu}(\cdot|s)||\fpi_0(\cdot|s)\right)},
\end{equation}
where $\DKL:\Delta_{\mathcal{A}}\times\Delta_{\mathcal{A}}\to\R_+$ with $\DKL\left(\fp||\fq\right) = \sum_a \fp(a) \log (\fp(a)/\fq(a))$ is the KL divergence or relative entropy and $\fpi_0\in\Pi$ is some some reference policy. The following corollary shows that, under Slater's condition, entropy and relative entropy regularization in the policy are indeed satisfying Assumption~\ref{ass:essential_smoothness}. By Corollary~\ref{cor:essential_smoothness_consequence} this implies that the optimal occupancy measure lies in the relative interior of $\mathcal{M}$.
\begin{corollary}\label{app:cor:essential_smoothness}
    Let Assumption~\ref{ass:slater} hold. Then, the regularizations 
    \begin{align}
        f_1(\fmu) &= -\beta \,\E_{(s,a)\sim\fmu}\bs{H\left(\fpi^{\fmu}(\cdot|s)\right)},\\
        f_2(\fmu) &= \beta \,\E_{(s,a)\sim\fmu}\bs{\DKL\left(\fpi^{\fmu}(\cdot|s)||\fpi_0(\cdot|s)\right)}\nonumber,
    \end{align}
    both satisfy Assumption~\ref{ass:essential_smoothness}.
\end{corollary}
To prove Corollary~\ref{app:cor:essential_smoothness} we first provide a formula for the gradients in Proposition~\ref{app:prop:entropy_regularization} below.
\begin{proposition}\label{app:prop:entropy_regularization} 
Consider a differentiable policy regularization $\Omega_s:\Delta_{\mathcal{A}}\to\R$ that is additionally allowed to depend on the state $s$. Let $\fmu\in\relint\Delta_{\mathcal{S}\times\mathcal{A}}$. For $n>1$ and $f(\fmu) = \E_{(s,a)\sim\fmu}\bs{\Omega_s\br{\fpi^{\fmu}(\cdot|s)}}$ we have
\begin{equation}\label{app:eq:reg_grad}
    \dfrac{\partial f(\fmu)}{\partial \fmu(s', a')} = \Omega_{s'}\br{\fpi^{\fmu}(\cdot|s')} + \nabla \Omega_{s'}\br{\fpi^{\fmu}(\cdot|s')}(a') - \sum_a \fpi^{\fmu}(a|s')\nabla \Omega_{s'}\br{\fpi^{\fmu}(\cdot|s')}(a).
\end{equation}
In particular, for $f_1(\fmu) = -\beta \,\E_{(s,a)\sim\fmu}\bs{H\left(\fpi^{\fmu}(\cdot|s)\right)}$ and $f_2(\fmu) = \beta \,\E_{(s,a)\sim\fmu}\bs{\DKL\left(\fpi^{\fmu}(\cdot|s)||\fpi_0(\cdot|s)\right)}$ with $\beta>0$ and $\fpi_0>\bm 0$, we get the following gradients:
    \begin{enumerate}[(a)]
        \item For $n=1$, we have $\nabla f_1(\fmu) = \beta\br{\log \fmu + \ones_{m}}$ and $\nabla f_2(\fmu) = \beta\br{\log \frac{\fmu}{\fpi_0} + \ones_{m}}$.
        \item For $n>1$, we have $\nabla f_1(\fmu) = \beta\log \fpi^{\fmu}$ and $\nabla f_2(\fmu) = \beta\log\frac{\fpi^{\fmu}}{\fpi_0}$, where $\frac{\fpi^{\fmu}}{\fpi_0}\defeq \myvec{\frac{\fpi^{\fmu}(a_1|s_1)}{\fpi_0(a_1|s_1)},\hdots, \frac{\fpi^{\fmu}(a_m|s_n)}{\fpi_0(a_m|s_n)}}^\top$.
    \end{enumerate}
\end{proposition}
\begin{proof}[Proof of Proposition~\ref{app:prop:entropy_regularization}]
    We define the state occupancy measure $\fnu(s) \defeq \sum_{a}\fmu(s,a)$. The result then follows by naive differentiation. In particular, by the product rule we have
    \begin{align}\label{app:eq:reg_grad1}
        \dfrac{\partial f(\fmu)}{\partial \fmu(s', a')} = \Omega_{s'}\br{\fpi^{\fmu}(\cdot|s')}(a') + \fnu(s') \dfrac{\partial\Omega_{s'}\br{\fpi^{\fmu}(\cdot|s')}}{\partial \fmu(s', a')}.
    \end{align}
    Furthermore, it holds
    \begin{align}\label{app:eq:reg_grad2}
        \dfrac{\partial\Omega_{s'}\br{\fpi^{\fmu}(\cdot|s')}}{\partial \fmu(s', a')} &= \nabla\Omega_{s'}\br{\fpi^{\fmu}(\cdot|s')}^\top \dfrac{\partial\fpi^{\fmu}(\cdot|s')}{\partial \fmu(s', a')}\\
        &\stackrel{}{=} \sum_a \nabla\Omega_{s'}\br{\fpi^{\fmu}(\cdot|s')}(a) \dfrac{\delta_{a, a'} - \fpi^{\fmu}(a|s')}{\fnu(s')}\nonumber\\
        &\stackrel{}{=} \dfrac{1}{\fnu(s')}\br{\nabla\Omega_{s'}\br{\fpi^{\fmu}(\cdot|s')}(a') - \sum_a  \fpi^{\fmu}(a|s')\nabla\Omega_{s'}\br{\fpi^{\fmu}(\cdot|s')}(a) },\nonumber
    \end{align}
    where we used that $\fpi^{\fmu}(a|s) = \fmu(s, a)/\fnu(s)$ and $\delta_{a, a'}$ denotes the Kronecker delta with $\delta_{a, a'}=1$ if $a=a'$ and $\delta_{a, a'}=0$ otherwise. Hence,
    \begin{equation}
        \dfrac{\partial\fpi^{\fmu}(\cdot|s')}{\partial \fmu(s', a')}(a) =  \dfrac{\fnu(s')\delta_{a, a'} - \fmu(s',a)}{\fnu(s')^2} = \dfrac{\delta_{a, a'} - \fpi^{\fmu}(a|s')}{\fnu(s')}.
    \end{equation}
    Plugging \eqref{app:eq:reg_grad2} back into \eqref{app:eq:reg_grad1} yields
    \begin{equation}
        \dfrac{\partial f(\fmu)}{\partial \fmu(s', a')} = \Omega_{s'}\br{\fpi^{\fmu}(\cdot|s')}(a') + \nabla\Omega_{s'}\br{\fpi^{\fmu}(\cdot|s')}(a') - \sum_a  \fpi^{\fmu}(a|s')\nabla\Omega_{s'}\br{\fpi^{\fmu}(\cdot|s')}(a).
    \end{equation}
    Now, for the special cases $f_1$ and $f_2$ we have $f_1(\fmu) = \E_{(s,a)\sim\fmu}\bs{\Omega_1\br{\fpi^{\fmu}(\cdot|s)}}$ and $f_2(\fmu) = \E_{(s,a)\sim\fmu}\bs{\Omega_{2,s}\br{\fpi^{\fmu}(\cdot|s)}}$ for $\Omega_1(\fd) = -\beta H(\fd)$ and $\Omega_{2,s}(\fd) = \beta \DKL(\fd||\fpi_0(\cdot|s))$, respectively. Moreover, $\nabla \Omega_1(\fd) = \beta\br{\log \fd + \ones_{m}}$ and $\nabla \Omega_{2,s}(\fd) = \beta\br{\log \frac{\fd}{\fpi_0(\cdot|s)} + \ones_{m}}$. This proves $(a)$, since for $n=1$ we have $\fmu = \fpi^{\fmu}$ and $f_i = \Omega_i$ for $i=1,2$. Moreover, to show $(b)$ we plug the above gradients of the policy regularizations back into the formula \eqref{app:eq:reg_grad} which yields
    \begin{align}
        \dfrac{\partial f_1(\fmu)}{\partial \fmu(s', a')} &= \Omega_1\br{\fpi^{\fmu}(\cdot|s')} + \nabla \Omega_1\br{\fpi^{\fmu}(\cdot|s')}(a') - \sum_a \fpi^{\fmu}(a|s')\nabla \Omega_1\br{\fpi^{\fmu}(\cdot|s')}(a)\\
        &=-\beta H(\fpi^{\fmu}(\cdot|s')) + \beta \br{\log \fpi^{\fmu}(a'|s') + 1} - \sum_a \fpi^{\fmu}(a|s')\beta \br{\log \fpi^{\fmu}(a|s') + 1} \nonumber\\
        &= \beta \log \fpi^{\fmu}(a'|s'),\nonumber
    \end{align}
    and
    \begin{align}
        \dfrac{\partial f_2(\fmu)}{\partial \fmu(s', a')} &= \Omega_{2,s'}\br{\fpi^{\fmu}(\cdot|s')} + \nabla \Omega_{2,s'}\br{\fpi^{\fmu}(\cdot|s')}(a') - \sum_a \fpi^{\fmu}(a|s')\nabla \Omega_{2,s'}\br{\fpi^{\fmu}(\cdot|s')}(a)\\
        &=\beta \DKL(\fpi^{\fmu}(\cdot|s')||\fpi_0(\cdot|s')) + \beta \br{\log \dfrac{\fpi^{\fmu}(a'|s')}{\fpi_0(a'|s')} + 1} - \sum_a \fpi^{\fmu}(a|s')\beta \br{\log \dfrac{\fpi^{\fmu}(a|s')}{\fpi_0(a|s')} + 1} \nonumber\\
        &= \beta \log \dfrac{\fpi^{\fmu}(a'|s')}{\fpi_0(a'|s')},\nonumber
    \end{align}
    as desired.
\end{proof}
Before we can proceed with the proof of Corollary~\ref{app:cor:essential_smoothness}, we need to prove the following proposition showing that under Slater's condition the state occupancy measure can only be zero if the policy assigns zero probability to some state action pair.\looseness-1
\begin{proposition}\label{app:prop:vanishing_policy}
Let Assumption~\ref{ass:slater} hold. If $\fnu(s)\defeq\sum_a \fmu(s,a) = 0$ for some $s\in\mathcal{S}$, then $\fpi^{\fmu}(a'|s') = 0$ for some $(s',a')\in\mathcal{S}\times \mathcal{A}$.
\end{proposition}
\begin{proof}
    We show the contraposition: if $\fpi^{\fmu}>\bm 0$, then $\fnu>\bm 0$. To this end, let $\fpi^{\fmu}>\bm 0$ and note that due to Assumption~\ref{ass:slater} (Slater's condition) there is some $\bar{\fmu}\in\mathcal{M}$ such that $\bar{\fmu}>\bm 0$. Hence for any $s\in\mathcal{S}$ it holds that
    \begin{equation}
        \bar{\fnu}(s) = \sum_a \bar{\fmu}(s,a) = (1-\gamma) \sum_{t=0}^{\infty} \gamma^t\Pp_{\fnu_0}^{\fpi^{\bar{\fmu}}}(s_t = s) > 0.
    \end{equation}
    Therefore, there must exist some $T\in\N$ such that $\Pp_{\fnu_0}^{\fpi^{\bar{\fmu}}}(s_T = s)>0$. Let us fix such a $T$. 
    
    If $T=0$, then $\Pp_{\fnu_0}^{\fpi^{\bar{\fmu}}}(s_T = s)=\fnu_0(s)>0$. In particular, this implies that $\fnu(s)>0$.
    
    If $T>0$, we have
    \begin{equation}
        \Pp_{\fnu_0}^{\fpi^{\bar{\fmu}}}(s_T = s) = \sum_{\substack{s_0,\hdots, s_{T-1}\\a_0,\hdots, a_{T-1}}}\fnu_0(s_0) \prod_{t=1}^T \fpi^{\bar{\fmu}}(a_{t-1}|s_{t-1})\fP(s_t|s_{t-1}, a_{t-1})>0.
    \end{equation}
    This implies that there is at least one path $(s_0, a_0, \hdots, s_T)$ with non-zero probability under $\Pp_{\fnu_0}^{\fpi^{\bar{\fmu}}}$ i.e.
    \begin{equation}
        \fnu_0(s_0) \prod_{t=1}^T \fpi^{\bar{\fmu}}(a_{t-1}|s_{t-1})\fP(s_t|s_{t-1}, a_{t-1})>0.
    \end{equation}
    Moreover, the above product remains positive under each non-vanishing policy. This implies that $\Pp_{\fnu_0}^{\fpi^{\fmu}}(s_T = s)>0$ and thus $\fnu(s)>0$. Since the above proof holds for each $s\in\mathcal{S}$, we have proven that $\fnu>\bm 0$.
\end{proof}
Now, we are ready to prove Corollary~\ref{app:cor:essential_smoothness}.
\begin{proof}[Proof of Corollary~\ref{app:cor:essential_smoothness}]
    Both regularizations are differentiable in the relative interior of their domain $\R^{nm}_+$ (with the gradients provided in Proposition~\ref{app:prop:entropy_regularization}). Now, let $\br{\fmu_k}_{k\in\N}$ be a sequence in $\relint \R^{nm}_+$ converging to some occupancy measure $\fmu\in\relbd\mathcal{M}$ i.e. we have $\fmu_k(s,a)\to\fmu(s,a) = 0$ for some $(s,a)\in\mathcal{S}\times\mathcal{A}$. In case $\fnu(s)=\sum_a\fmu(s,a)>0$, this implies that $\fpi^{\fmu}(a|s) = 0$. Moreover, in case $\fnu(s)= 0$, the result of Proposition~\ref{app:prop:vanishing_policy} implies that we have $\fpi^{\fmu}(a'|s') = 0$ for some other $(s',a')\in\mathcal{S}\times\mathcal{A}$. Now, since the mapping
    \begin{equation}
    \fmu\mapsto\fpi^{\fmu}(a|s) = \begin{cases} \fmu(s,a)/\fnu(s)&, \fnu(s)>0\\ 1/|\mathcal{A}|&, \text{ otherwise,} \end{cases}
    \end{equation}
    is for all state-action pairs continuous on $\bc{\fmu\in\mathcal{M}: \fpi^{\fmu}(a|s)=0}$, convergence in occupancy measure $\fmu_k(s,a)\to\fmu(s,a) = 0$ implies $\fpi^{\fmu_k}(a'|s')\to\fpi^{\fmu}(a'|s') = 0$ for some $(s',a')\in\mathcal{S}\times\mathcal{A}$. Therefore, since $|\log(x)|\to \infty$ as $x\to 0$, we have $\lim_{k\to\infty} \norm{\nabla f_i(\fmu_k)} = \infty$ for $i=1,2$.
\end{proof}

Next, we provide the proof of Proposition~\ref{prop:convexity} showing that policy regularization is a special case of occupancy measure regularization.

\subsection{Proof of Proposition~\ref{prop:convexity}}\label{app:subsec:prop:convexity}
\textbf{Proposition~\ref{prop:convexity}}
\textit{
    Let $f(\fmu)=\E_{(s,a)\sim\fmu}\bs{\Omega\left(\fpi^{\fmu}(\cdot|s)\right)}$. \vspace{-0.3cm}
    \begin{enumerate}[(a)]
        \item If $\Omega$ is convex, then so is $f$.\vspace{-0.1cm}
        \item If $\Omega$ is strictly convex, then so is $f$.
    \end{enumerate}\vspace{-0.2cm}
}
\begin{proof} 
Defining $\fmu_s\defeq \myvec{\fmu(s,a_1),\hdots,\fmu(s,a_m)}^\top$ and denoting the all-one vector in $\R^m$ by $\f{1}$ we can rewrite
\begin{equation}
	f(\fmu)=\sum_{s:\f{1}^\top\fmu_s>0} \f{1}^\top\fmu_s\Omega\br{\dfrac{\fmu_s}{\f{1}^\top\fmu_s}}.
\end{equation}
To prove (strict) convexity consider $\fmu,\bar{\fmu}\in \R^{nm}_+$ with $\fmu\neq\bar{\fmu}$. It will be convenient to define the sets $\mathcal{V}\defeq\bc{s\in\mathcal{S}: \f{1}^\top\fmu_s>0}$ and $\mathcal{W}\defeq\bc{s\in\mathcal{S}: \f{1}^\top\bar{\fmu}_s > 0}$. Let $\alpha\in(0,1)$ and $\bar{\alpha} \defeq 1-\alpha$, then it follows from $\mathcal{V}\cup\mathcal{W} = (\mathcal{V}\cap\mathcal{W})\cup (\mathcal{V}\setminus\mathcal{W})\cup(\mathcal{W}\setminus\mathcal{V})$ that
\begin{align}\label{app:eq:convexity_proof}
	&f(\alpha\fmu+ \bar{\alpha}\bar{\fmu})\\
	&= \sum_{s \in\mathcal{V}\cup\mathcal{W}} \f{1}^\top\br{\alpha\fmu_s + \bar{\alpha}\bar{\fmu}_s}\Omega\br{\dfrac{\alpha\fmu_s + \bar{\alpha}\bar{\fmu}_s}{\f{1}^\top\br{\alpha\fmu_s + \bar{\alpha}\bar{\fmu}_s}}}\nonumber\\
	&= \underbrace{\sum_{s\in\mathcal{V}\cap\mathcal{W}} \f{1}^\top\br{\alpha\fmu_s + \bar{\alpha}\bar{\fmu}_s}\Omega\br{\dfrac{\alpha\fmu_s + \bar{\alpha}\bar{\fmu}_s}{\f{1}^\top\br{\alpha\fmu_s + \bar{\alpha}\bar{\fmu}_s}}}}_{(\Delta)} + \sum_{s\in\mathcal{V}\setminus\mathcal{W}} \alpha\f{1}^\top\fmu_s \Omega\br{\dfrac{\alpha\fmu_s}{\alpha\f{1}^\top\fmu_s}}\nonumber\\
	&+ \sum_{s\in\mathcal{W}\setminus\mathcal{V}} \bar{\alpha}\f{1}^\top\bar{\fmu}_s \Omega\br{\dfrac{\bar{\alpha}\bar{\fmu}_s}{\bar{\alpha}\f{1}^\top\bar{\fmu}_s}}.\nonumber
\end{align}
From here on we can use (strict) convexity of $\Omega$ to bound $(\Delta)$ as follows
\begin{align}
    (\Delta) &= \sum_{s\in\mathcal{V}\cap\mathcal{W}} \f{1}^\top\br{\alpha\fmu_s + \bar{\alpha}\bar{\fmu}_s}\Omega\br{\dfrac{\alpha\f{1}^\top\fmu_s}{\f{1}^\top\br{\alpha\fmu_s + \bar{\alpha}\bar{\fmu}_s}}\dfrac{\fmu_s}{\f{1}^\top\fmu_s}+\dfrac{\bar{\alpha}\f{1}^\top\bar{\fmu}_s}{\f{1}^\top\br{\alpha\fmu_s + \bar{\alpha}\bar{\fmu}_s}}\dfrac{\bar{\fmu}_s}{\f{1}^\top\bar{\fmu}_s}}\\
	&\stackrel{(<)}{\leq} \sum_{s\in\mathcal{V}\cap\mathcal{W}} \f{1}^\top\br{\alpha\fmu_s + \bar{\alpha}\bar{\fmu}_s}\br{\dfrac{\alpha\f{1}^\top\fmu_s}{\f{1}^\top\br{\alpha\fmu_s + \bar{\alpha}\bar{\fmu}_s}}\Omega\br{\dfrac{\fmu_s}{\f{1}^\top\fmu_s}}+\dfrac{\bar{\alpha}\f{1}^\top\bar{\fmu}_s}{\f{1}^\top\br{\alpha\fmu_s + \bar{\alpha}\bar{\fmu}_s}}\Omega\br{\dfrac{\bar{\fmu}_s}{\f{1}^\top\bar{\fmu}_s}}}\nonumber\\
	&= \sum_{s\in\mathcal{V}\cap\mathcal{W}} \br{\alpha\f{1}^\top\fmu_s\Omega\br{\dfrac{\fmu_s}{\f{1}^\top\fmu_s}}+\bar{\alpha}\f{1}^\top\bar{\fmu}_s\Omega\br{\dfrac{\bar{\fmu}_s}{\f{1}^\top\bar{\fmu}_s}}}\nonumber.
\end{align}
Plugging this back into \eqref{app:eq:convexity_proof} yields (strict) convexity as desired
\begin{align}
    &f(\alpha\fmu+ \bar{\alpha}\bar{\fmu})\\
    &\stackrel{(<)}{\leq} \sum_{s\in\mathcal{V}\cap\mathcal{W}} \br{\alpha\f{1}^\top\fmu_s\Omega\br{\dfrac{\fmu_s}{\f{1}^\top\fmu_s}}+\bar{\alpha}\f{1}^\top\bar{\fmu}_s\Omega\br{\dfrac{\bar{\fmu}_s}{\f{1}^\top\bar{\fmu}_s}}}+\sum_{s\in\mathcal{V}\setminus\mathcal{W}} \alpha\f{1}^\top\fmu_s \Omega\br{\dfrac{\fmu_s}{\f{1}^\top\fmu_s}}\nonumber\\
	&+ \sum_{s\in\mathcal{W}\setminus\mathcal{V}} \bar{\alpha}\f{1}^\top\bar{\fmu}_s \Omega\br{\dfrac{\bar{\fmu}_s}{\f{1}^\top\bar{\fmu}_s}}\nonumber\\
    &= \alpha f(\fmu) + \bar{\alpha} f(\bar{\fmu}),\nonumber
\end{align}
where we used $\mathcal{V} = (\mathcal{V}\cap\mathcal{W})\cup(\mathcal{V}\setminus\mathcal{W})$ and $\mathcal{W} = (\mathcal{V}\cap\mathcal{W})\cup(\mathcal{W}\setminus\mathcal{V})$ in the last equality.
\end{proof}

\subsection{Proof of Proposition~\ref{prop:strong_duality}}\label{app:subsec:prop:strong_duality}
\textbf{Proposition~\ref{prop:strong_duality}}
\textit{
If Assumption~\ref{ass:slater} and \ref{ass:strict_convexity} hold, the dual optimum of \eqref{eq:cmdp_occ_dual} is attained for some $\fxi^*\geq\bm 0$, and \eqref{eq:cmdp_occ} is equivalent to an unconstrained MDP problem of reward $\fr-\fPsi\fxi^*$. In other words, it holds
    \begin{equation}\label{app:eq:equivalence_constrained_unconstrained_MDP}
    \RL(\fr) = \RLO(\fr-\fPsi\fxi^*).
    \end{equation}
}
\begin{proof}
The proof is based on standard Lagrangian duality theory. First, we note that the CMDP problem \eqref{eq:cmdp_occ} is a convex optimization problem. Its primal optimum is finite, as the feasible set $\mathcal{F}\subseteq\Delta_{\mathcal{S}\times\mathcal{A}}$ is bounded and the objective is upper bounded by a linear function (since $f$ is convex). From Slater's condition it follows that strong duality holds and the dual optimum is attained by some not necessarily unique $\fxi^*\f{\geq 0}$ \citep{boyd2004convex}. 

To show equation \eqref{app:eq:equivalence_constrained_unconstrained_MDP}, note that the primal optimum $\RL(\fr) = \bc{\fmu^*}$ is unique due to strict convexity of $f$. Moreover, for each pair $\br{\fmu^*, \fxi^*}$ of primal and dual optimal solutions the Lagrangian 
\begin{equation}
    L(\fmu, \fxi) = \fr^\top \fmu - f(\fmu) + \fxi^\top\br{\fb-\fPsi^\top\fmu},
\end{equation}
has a saddle point at $\br{\fmu^*, \fxi^*}$ i.e.
\begin{equation}
    L(\fmu, \fxi^*) \leq L(\fmu^*, \fxi^*) \leq L(\fmu^*, \fxi), \quad \forall \fmu\in\mathcal{M}, \fxi\geq \f{0}.
\end{equation}
We then have $\RLO(\fr-\fPsi\fxi^*) = \argmax_{\fmu\in\mathcal{M}} L(\fmu, \fxi^*) = \bc{\fmu^*}$ where we again used strict convexity of $f$ for the last equality.
\end{proof}

\subsection{Remarks on Strong Duality}\label{app:subsec:remarks_strong_duality}
Whereas strong duality holds also for unregularized CMDPs \citet{altman1999constrained}, unique recovery of the optimal occupancy measure from an unconstrained RL problem \eqref{app:eq:equivalence_constrained_unconstrained_MDP} is a consequence of the strictly convex regularization. To illustrate this, consider the following simple example.
\begin{example}\label{ex:running_example}
    Consider a single state MDP with $\mathcal{A}=\bc{a_1, a_2}$. The reward is defined via $\fr=\myvec{\fr(a_1), \fr(a_2)}^\top = \myvec{0, 1}^\top$ and there is no regularization. Furthermore, the agent needs to respect the constraints $\fPsi^\top\fmu = \myvec{0, 1}^\top\fmu = \fmu(a_2) \leq 3/4$. In this single state setting $\fmu^{\fpi}(a)=\fpi(a)$ and $\mathcal{M} = \Delta_{\mathcal{A}}$. Clearly, the unique primal optimal solution is $\fmu^*(a_1) = 1/4$ and $\fmu^*(a_2) = 3/4$. This is a key difference to the unconstrained setting where always a deterministic optimal policy exists. Thus, $\fmu^*$ cannot be realized as the unique optimum of an unconstrained, unregularized MDP, but only as the convex combination of multiple deterministic solutions. Indeed relaxing the safety constraint yields the Lagrangian $L(\fmu, \xi) = (\fr - \fPsi\xi)^\top \fmu + \xi b = (1-\xi)\fmu(a_2) + 3\xi/4 $ and the dual function
\begin{equation}
    g(\xi) = \max_{\fmu\in\Delta_{\mathcal{A}}} L(\fmu, \xi) = \begin{cases} 1 - \xi/4 \quad &,\xi \leq 1\\
    3\xi/4 \quad &,\xi>1.\end{cases}
\end{equation}
Thus, there is a unique dual optimum $\xi^* = \argmin_{\xi\geq 0} g(\xi) = 1$, leading to the dual optimal value $1/2$, which is equal to the primal optimum due to strong duality. However, for the reward $\fr-\fPsi\xi = \myvec{0, 0}^\top$ not only $\fmu^*$, but all $\fmu\in\Delta_{\mathcal{A}}$ are optimal in the unconstrained problem -- even those with $\fmu(a_2) > 3/4$ that are primal infeasible.
\end{example}

\section{Proofs and Comments of Section~\ref{sec:constrained_IRL}}\label{app:sec:identifiability}
\subsection{Preliminaries from Convex Analysis}
Throughout this section, we introduce a few additional tools from convex analysis which turn out to be useful for the proof of Theorem~\ref{thm:identifiability}. In convex analysis it is standard to extend convex functions over the entire space by setting their value to $+\infty$ outside of their domain. This leads us to extended real value functions $h:\R^{n}\to[-\infty, \infty]$. Their effective domain is defined as $\dom h \defeq \bc{\fx : h(\fx)< \infty}$, and a convex function $h$ is said to be proper if $h>-\infty$ and $\dom h \neq \emptyset$. Furthermore, $h$ is referred to as closed if its epigraph $\bc{\br{\fx, y}:\fx\in\dom h, y\geq h(\fx)}$ is a closed set. For instance, $h$ is closed if it is continuous and $\dom h$ is a closed set \citep{boyd2004convex}. Next, we introduce the two key tools needed for the proof of Theorem~\ref{thm:identifiability} -- convex conjugates and the Moreau-Rockafeller theorem.

\paragraph{Convex Conjugate} The convex conjugate $h^*:\R^{n}\to[-\infty, \infty]$ of $h$ is defined as
\begin{equation}
    h^*(\fy) \defeq \sup_{\fx\in\R^{n}} \fy^\top \fx - h(\fx).
\end{equation}
If $h$ is closed proper convex, it holds $h^{**}=h$. Moreover, the following optimality conditions hold.
\begin{theorem}[\citet{rockafellar1970convex}]\label{app:thm:convex_conjugate}
    For any proper convex function $h^*:\R^{n}\to[-\infty, \infty]$ it holds
    \begin{equation}
        h^*(\fy) = \fy^\top \fx - h(\fx) \quad \iff \quad \fy \in \partial h(\fx).
    \end{equation}
    If additionally $h$ is closed, then
    \begin{equation}
        h^*(\fy) = \fy^\top \fx - h(\fx) \quad \iff \quad \fy \in \partial h(\fx) \quad \iff \quad \fx \in \partial h^*(\fy).
    \end{equation}
\end{theorem}

\paragraph{Moreau-Rockafeller Theorem} The following theorem gives sufficient conditions under which the sum of subdifferentials of two functions is equal to the subdifferential of the sum of the two functions.
\begin{theorem}[\citet{rockafellar1970convex}]\label{app:thm:moreau-rockafeller}
    Let $h_1, h_2$ be proper convex functions on $\R^n$ and let $h\defeq h_1 + h_2$. If $\relint (\dom h_1)$ and $\relint (\dom h_2)$, have a point in common then
    \begin{equation}
        \partial h(\fx) = \partial h_1(\fx) + \partial h_2(\fx),\, \forall \fx.
    \end{equation}
    If $h_1$ is polyhedral (i.e. its epigraph is polyhedral) then it is enough if the sets $\dom h_1$ and $\relint (\dom h_2)$ have a point in common.
\end{theorem}

\subsection{Proof of Proposition~\ref{prop:consistency}}\label{app:subsec:prop:consistency}
\textbf{Proposition~\ref{prop:consistency}}
\textit{
If Assumption~\ref{ass:realizability} holds, then the rewards optimizing
\begin{equation}\label{app:eq:cirl_occ_exact}\tag{IRL}
    \min_{\fr\in\mathcal{R}}\max_{\fmu\in\mathcal{F}}\;\;\fr^\top\br{\fmu - \fmuE} - f(\fmu),
\end{equation}
are exactly those rewards in $\mathcal{R}$ for which the expert occupancy measure is optimal in problem \eqref{eq:cmdp_occ}.
}
\begin{proof}
We can rewrite problem \eqref{app:eq:cirl_occ_exact} equivalently as $\min_{\fr\in\mathcal{R}}\max_{\fmu\in \mathcal{F}}L(\fmu,\fr)$, where $L(\fmu,\fr)\defeq J(\fmu,\fr) - J(\fmuE,\fr)$ and $J(\fmu,\fr) \defeq \fr^\top \fmu- f(\fmu)$. For a fixed $\fr$ it clearly holds $\argmax_{\fmu\in \mathcal{F}}L(\fmu,\fr) =  \RL(\fr)$. Also, we always get the lower bound $\max_{\fmu\in \mathcal{F}}L(\fmu,\fr)\geq 0$. This lower bound is achieved if and only if $\fmuE\in\argmax_{\fmu\in \mathcal{F}}L(\fmu,\fr) =  \RL(\fr)$. By Assumption~\ref{ass:realizability}, there is indeed $\frE\in\mathcal{R}$ such that $\fmuE\in\RL(\frE)$, and thus $\max_{\fmu\in \mathcal{F}}L(\fmu, \frE)= 0$. Therefore, any optimal $r^*\in\mathcal{R}$ must achieve $\max_{\fmu\in \mathcal{F}}L(\fmu, \fr^*)= 0$, which implies $\fmuE\in\RL(\fr^*)$. Moreover, for any $\fr^*\in\mathcal{R}$ with $\fmuE\in\RL(\fr^*)$ it needs to hold $\max_{\fmu\in \mathcal{F}}L(\fmu, \fr^*)= 0$, which proves optimality of $\fr^*$.
\end{proof}
\subsection{Proof of Theorem~\ref{thm:identifiability}}\label{app:subsec:thm:identifiability}
\textbf{Theorem~\ref{thm:identifiability}}
\textit{
    Let Assumption~\ref{ass:slater} hold and consider $\fmu\in\mathcal{F}$. Let $\mathcal{I}(\fmu)$ and $\mathcal{J}(\fmu)$ denote the set of indices of active inequality constraints under $\fmu$ i.e. $\fPsi_{i}^\top\fmu=\fb_i$ and $\fmu(s,a) = 0$ if and only if $i\in\mathcal{I}(\fmu)$ and $(s,a)\in\mathcal{J}(\fmu)$. Then,
    \begin{equation}\label{app:eq:identifiability_general}
        \fmu\in \RL(\fr) \iff \fr\in\partial f(\fmu) + N_{\mathcal{F}}(\fmu),
    \end{equation}
    where $N_{\mathcal{F}}(\fmu) = \mathcal{U} + \mathcal{C}(\fmu) + \mathcal{E}(\fmu)$ with 
    \begin{align}
        \mathcal{C}(\fmu)&\defeq \cone\br{\bc{\fPsi_i}_{i\in \mathcal{I}(\fmu)}},\nonumber\\
        \mathcal{E}(\fmu)&\defeq \cone\br{\bc{-\fe_{s,a}}_{(s,a)\in\mathcal{J}(\fmu)}}.\nonumber
    \end{align}
    Here, $\fe_{s,a}\in\R^{nm}$ denote the standard unit vectors with $\fe_{s,a}(s',a')=1$ if $(s,a)=(s',a')$ and $\fe_{s,a}(s',a')=0$ otherwise.
}
\begin{proof}
    The main idea of the proof is to use Theorem~\ref{app:thm:convex_conjugate} and \ref{app:thm:moreau-rockafeller} to prove that $\fmu\in\mathcal{F}$ is optimal for some $\fr$ if and only if $\fr\in\partial f(\fmu) + N_{\mathcal{F}}(\fmu)$. In order to apply Theorem~\ref{app:thm:convex_conjugate} and \ref{app:thm:moreau-rockafeller} to the constrained MDP problem, we recall that $f:\mathcal{X}\to\R$ is by definition a continuous convex function with $\mathcal{X}\subseteq\R^{nm}$ closed convex. We define the extended real value functions
\begin{equation}
    \bar{f}:\R^{nm}\to [-\infty, \infty], \; \fmu\mapsto \bar{f}(\fmu) \defeq \begin{cases}
        f(\fmu), &\fmu\in\mathcal{X},\\
        \infty, &\fmu\notin\mathcal{X},
    \end{cases}
\end{equation}
and 
\begin{equation}
    g_{\mathcal{F}}:\R^{nm}\to [-\infty, \infty], \; \fmu\mapsto g_{\mathcal{F}}(\fmu)\defeq\bar{f}(\fmu) + \delta_{\mathcal{F}}(\fmu),
\end{equation}
where $\delta_{\mathcal{F}}$ is the characteristic function
\begin{equation}
    \delta_{\mathcal{F}} (\fmu) \defeq \begin{cases}
        0 \quad &, \fmu\in\mathcal{F}\\
        \infty \quad &, \fmu\notin\mathcal{F}.
    \end{cases}
\end{equation}
Note that since $f$ is continuous and $\mathcal{F}$ closed, $g_{\mathcal{F}}$ is a closed proper convex function. Now, we can rewrite the CMDP problem \eqref{eq:cmdp_occ} as 
\begin{equation}
    \max_{\fmu\in\mathcal{F}}  \fr^\top \fmu - \Bar{f}(\fmu) =\max_{\fmu\in\R^{nm}} \fr^\top \fmu - g_{\mathcal{F}}(\fmu)= g_{\mathcal{F}}^*(\fr),
\end{equation}
which is exactly taking the form of the convex conjugate of $g_{\mathcal{F}}$.\footnote{Since the maximum is achieved here, we can replace the supremum with the maximum.} Therefore, Theorem~\ref{app:thm:convex_conjugate} yields
\begin{equation}\label{app:eq:opt_conditions_convex_conjugate}
    \fmu \in \RL(\fr) = \argmax_{\fmu\in\mathcal{F}}\fr^\top \fmu - g_{\mathcal{F}}(\fmu) \iff \fr \in \partial g_{\mathcal{F}}(\fmu).
\end{equation}

Since Slater's condition is satisfied we have $\relint(\dom \delta_{\mathcal{F}})\cap\relint(\dom \Bar{f}) = \relint\mathcal{F}\cap \relint\mathcal{X} = \relint\mathcal{F}\neq \emptyset$. Hence, the conditions of Theorem~\ref{app:thm:moreau-rockafeller} are satisfied and we get
\begin{equation}
    \partial g_{\mathcal{F}} (\fmu) = \partial \Bar{f}(\fmu) + \partial \delta_{\mathcal{F}} (\fmu).
\end{equation}
Using that $\partial \Bar{f} = \partial f$ and $\partial \delta_{\mathcal{F}}(\fmu) = N_{\mathcal{F}}(\fmu)$ \citep{rockafellar1970convex} we arrive at
\begin{equation}
    \fr\in\IRLO(\fmu) \iff  \fr\in \partial f(\fmu) + N_{\mathcal{F}}(\fmu).
\end{equation}
To finish the proof we note that for the polyhedron $\mathcal{F}$ the normal cone takes the form 
\begin{equation}
     N_{\mathcal{F}}(\fmu) = \spn\br{\fE-\gamma\fP} + \cone\br{\bc{\fPsi_i}_{i\in \mathcal{I}(\fmu)}} + \cone\br{\bc{-\fe_{s,a}}_{(s,a)\in\mathcal{J}(\fmu)}},
\end{equation}
where $\mathcal{I}(\fmu)$ and $\mathcal{J}(\fmu)$ are the sets of active safety and non-negativity constraints, respectively \citep{rockafellar2009variational}.
\end{proof}
\begin{remark}
    Note that as a consequence of \eqref{app:eq:opt_conditions_convex_conjugate} it holds $\IRLO(\fmu) = \partial g_{\mathcal{F}}(\fmu)$ and since $g_{\mathcal{F}}$ is closed proper convex Theorem~\ref{app:thm:convex_conjugate} also implies $\RL(\fr) = \partial g_{\mathcal{F}}^*(\fr)$.
\end{remark}

\subsection{Identifiability for State-Action-State Rewards}\label{app:subsec:state-action-state-rewards}
Throughout this paper our focus lies on state-action rewards $\fr:\mathcal{S}\times\mathcal{A}\to \R, (s,a)\mapsto \fr(s,a)$ that are naturally arising in the convex analytic approach to CMDPs (see \eqref{eq:cmdp_occ} and \citep{altman1999constrained}), and as the dual variables to the occupancy measure matching problem (see \eqref{eq:cirl_occ_emp} and \citep{ho2016generative}). However, some authors \citep{ng1999policy,sutton2018reinforcement, skalse2022invariance} also consider state-action-state rewards $\bar{\fr}:\mathcal{S}\times\mathcal{A}\times \mathcal{S}\to \R, (s,a,s')\mapsto \bar{\fr}(s,a,s')$ that are allowed to depend on the consecutive state $s'\sim \fP(\cdot|s,a)$. As mentioned in Section~\ref{sec:constrained_IRL}, this adds no generality to the forward CMDP problem, since a CMDP problem with state-action-state reward $\bar{\fr}$ is equivalent to a CMDP problem with the state-action reward $\fr(s,a)\defeq \E_{s'\sim \fP(\cdot|s,a)} \bar{\fr}(s,a,s')$. Nevertheless, in practice the transition law is typically unknown and it may in certain cases be easier to specify a state-action-state reward. To relate our identifiability results to this setting, we make use of the following vector notation for a state-action-state reward $\bar{\fr}(s,a,s')$
\begin{equation}
    \bar{\fr} = \myvec{\bar{\fr}_{s'_1}^\top,\hdots, \bar{\fr}_{s'_n}^\top}^\top\in\R^{n^2m}, \; \text{ with } \bar{\fr}_{s'} = \bar{\fr}(\cdot, \cdot, s')\in \R^{nm},
\end{equation}
and define the linear mapping $\mathbf{A}:\R^{n^2m} \to \R^{nm}$ via $(\bm A \bar{\fr})(s,a) \defeq \E_{s'\sim \fP(\cdot|s,a)} \bar{\fr}(s,a,s')$. Furthermore, we denote $\overline{\mathsf{RL}}_{\mathcal{F}}(\Bar{\fr}) \defeq \RL(\bm A \Bar{\fr})$ for the CMDP solution map for state-action-state rewards. The following corollary shows that identifiability of state-action-state rewards can be reduced to identifiability of state-action rewards -- and hence to the result of Theorem~\ref{thm:identifiability}.
\begin{corollary}\label{app:cor:sas_rewards}
    Let Assumption~\ref{ass:slater} hold and consider $\fmu\in\mathcal{F}$. Then,
    \begin{equation}
        \fmu\in \overline{\mathsf{RL}}_{\mathcal{F}}(\bar{\fr}) \iff \bar{\fr}\in \bc{\myvec{\fr\\\vdots\\ \fr} \bigg\vert \fr\in \partial f(\fmu) + \mathcal{U} + \mathcal{C}(\fmu) + \mathcal{E}(\fmu)} + \mathcal{V},
    \end{equation}
    where $\mathcal{V} = \ker\bm A$ with $\dim \mathcal{V} =n(nm-1)$.
\end{corollary}
\begin{proof}
    By Theorem~\ref{thm:identifiability}, we have $\fmu\in \overline{\mathsf{RL}}_{\mathcal{F}}(\bar{\fr})=\RL(\bm A \bar{\fr})$ if and only if $\bm A \bar{\fr}\in\partial f(\fmu) + \mathcal{U} + \mathcal{C}(\fmu) + \mathcal{E}(\fmu)$. It therefore suffices to show that $\bm A \bar{\fr} = \fr$ if and only if $\bar{\fr} = \bar{\fr}' + \bar{\fr}''$ with $\bar{\fr}'=[\fr, \hdots, \fr]^\top$ and $\bar{\fr}''\in \ker \bm A$. If $\bar{\fr} = \bar{\fr}' + \bar{\fr}''$ with $\bar{\fr}'=[\fr, \hdots, \fr]^\top$ and $\bar{\fr}''\in \ker \bm A$, then it follows from $\bm A \bar{\fr}' = \fr$ that $\bm A \bar{\fr} = \fr + \bm 0$. Conversely, if $\bm A \bar{\fr} = \fr$, then $\bar{\fr}'=[\fr, \hdots, \fr]^\top$ also satisfies $\bm A \bar{\fr}' = \fr$, and thus $\bar{\fr}-\bar{\fr}'\in\ker \bm A$.

    Finally, since for $\bar{\fr} = [\fr, \hdots, \fr]^\top$ we have $\bm A \bar{\fr} = \fr$, the mapping $\bm A$ is surjective and thus $\dim \mathcal{V} =n(nm-1)$.
\end{proof}
From a more abstract perspective, Corollary~\ref{app:cor:sas_rewards} makes use of the fact that the image $\textbf{im} \bm A = \R^{nm}$ of $\bm A$ is isomorphic to the quotient space $\R^{n^2m}/ \ker \bm A$ \citep{halmos2017finite}, where $\R^{n^2m}/ \ker \bm A$ is the set of equivalence classes $[\bar{\fr}'] \defeq \bc{\bar{\fr}\in\R^{n^2m}: \bar{\fr} = \bar{\fr}' + \bar{\fr}'',  \bar{\fr}''\in \ker \bm A}$.

For unconstrained MDPs \citet{skalse2022invariance} discuss invariances of optimal policies to reward transformation for state-action-state rewards. They introduce the additional invariances along the linear subspace $\mathcal{V}=\ker \mathbf{A}$ as \emph{$s'$-redistribution}. Moreover, they show that identifying a state-action-state reward up to $s'$-redistribution is not sufficient for generalizability to new environments. In contrast, our result in Theorem~\ref{thm:intersection} shows that for state-action rewards identifying the rewards up to potential shaping is not enough for generalizability and that we instead need to recover the expert's reward up to a constant. However, this result is not extending immediately to the state-action-state setting, and one would need to modify the proof of Theorem~\ref{thm:intersection} to additionally account for the space $\mathcal{V}$ in order to make a statement about generalizability.

\subsection{Proof of Corollary~\ref{cor:essential_smoothness_consequence}}\label{app:subsec:cor:essential_smoothness_consequence}
\textbf{Assumption~\ref{ass:essential_smoothness}}
    Let $f:\mathcal{X}\to \R$ be such that:\vspace{-0.3cm}
\begin{enumerate}[(a)]
    \item $f$ is differentiable throughout $\interior\mathcal{X}$,
    \item $\lim_{k\to \infty} \norm{\nabla f(\fmu_k)}= \infty$ if $\br{\fmu_k}_{k\in\mathbb{N}}$ is a sequence in $\interior\mathcal{X}$ converging to a point $\fmu\in \relbd \mathcal{M}$.
\end{enumerate}
\textbf{Corollary~\ref{cor:essential_smoothness_consequence}}
\textit{
    Let Assumptions~\ref{ass:slater}, \ref{ass:realizability}, \ref{ass:essential_smoothness} hold. Then, we have $\RL(\fr) \subset \relint\mathcal{M}$ for any $\fr\in\R^{nm}$ and
    \begin{equation}
        \IRLO(\fmuE) = \nabla f(\fmuE) + \mathcal{U} + \mathcal{C}(\fmuE).
    \end{equation}
}
\begin{proof}
    Under Assumption~\ref{ass:slater}, Theorem~\ref{thm:identifiability} states that
    \begin{equation}\label{app:eq:id_essential_smoothness}
        \fmu \in \RL(\fr) \iff \fr \in \partial f(\fmu) + N_{\mathcal{F}}(\fmu).
    \end{equation}
    However, Assumption~\ref{ass:essential_smoothness} ensures that $f(\fmu)=\bc{\nabla f(\fmu)}$ for $\fmu\in\relint\mathcal{M}$ and $\partial f(\fmu) = \emptyset$ for $\fmu\in\relbd\mathcal{M}$ (see \citep[Theorem 25.6.]{rockafellar1970convex}). Hence,
    \begin{equation}
        \fmu \in \RL(\fr)\subset \relbd\mathcal{M} \iff \fr \in \emptyset + N_{\mathcal{F}}(\fmu) = \emptyset.
    \end{equation}
    Furthermore, under Assumption~\ref{ass:realizability}, it follows from differentiability in $\relint{\mathcal{M}}$ and Corollary~\ref{cor:identifiability_irl} that
    \begin{equation}
        \IRLO(\fmuE) = \nabla f(\fmuE) + \mathcal{U} + \mathcal{C}(\fmuE).
    \end{equation}
\end{proof}

\subsection{Proof of Theorem~\ref{thm:intersection}}\label{app:subsec:thm:intersection}
\textbf{Theorem~\ref{thm:intersection}}
\textit{
        Let Assumption~\ref{ass:slater}, \ref{ass:strict_convexity}, \ref{ass:realizability}, \ref{ass:essential_smoothness} be satisfied for $(\fP_0,\fb_0)$ and let $\fmuE\in\RL^{\fP_0,\fb_0}(\frE)$ for some $\frE\in\mathcal{R}$. Consider an arbitrary neighborhood $\mathcal{O}_{\fP_0}\subseteq \R^{nm\times n}$ of $\fP_0$. Then, IRL generalizes to $\mathcal{P} = \mathcal{O}_{\fP_0}\cap\mathfrak{P}$ and $\mathcal{B}=\R^k$ if and only if
    \begin{equation}
        \IRL^{\fP_0, \fb_0}(\fmuE) \subseteq \frE + \spn(\ones_{nm}).
    \end{equation}
}

\begin{proof}
    The \emph{if} direction is trivial, since addition of a constant is not changing the set of optimal occupancy measures. Hence, if $\IRL^{\fP_0, \fb_0}(\fmuE) \subseteq \frE + \spn(\ones_{nm})$, then IRL generalizes to any arbitrary set of transition laws and constraint thresholds.

    To prove the \emph{only if} direction, we proceed in the following steps:
    \begin{enumerate}
    \item Show that Slater's condition is still satisfied in a sufficiently small neighborhood of $\fP_0$.
    \item Apply Theorem~\ref{thm:identifiability} to rewrite generalizability as a condition on the rewards.
    \item Construct $\bar{\fP}_1, \bar{\fP}_2\in\mathfrak{P}$ such that only $\frE+ \spn\br{\mathbf{1}_{nm}}$ generalize to $\bar{\fP}_1, \bar{\fP}_2$.
    \item Use $\bar{\fP}_1, \bar{\fP}_2$ to construct $\fP_1, \fP_2\in\mathcal{P}$ such that only $\frE+ \spn\br{\mathbf{1}_{nm}}$ generalize to $\fP_1, \fP_2$.
\end{enumerate}

    \textbf{Step 1:}\\
    By Assumption~\ref{ass:slater} (Slater's condition) there is some occupancy measure $\bar{\fmu}\in\relint \mathcal{F}^{\fP_0, \fb_0}$ where
    \begin{equation}
        \mathcal{F}^{\fP_0, \fb_0} \defeq \bc{\fmu\in\R^{nm}: \fmu\geq \bm 0, (\fE-\gamma\fP_0)^\top\fmu = (1-\gamma)\fnu_0, \fPsi^\top \fmu \leq \fb_0}.
    \end{equation}
    Note that $\mathcal{F}^{\fP_0, \fb_0}\subseteq\mathcal{F}^{\fP_0, \fb}$ for $\fb\geq \fb_0$. Thus, Slater's condition remains to hold when the constraint threshold is relaxed. Furthermore, for $\fP\in\mathfrak{P}$ consider the orthogonal projection
    \begin{equation}
        \operatorname{Proj}^{\fP}(\bar{\fmu}) \defeq \argmin_{\fmu: (\fE  - \gamma \fP)^\top\fmu = (1-\gamma)\fnu_0} \norm{\fmu - \bar{\fmu}}_2,
    \end{equation}
    of $\bar{\fmu}$ onto the affine hull of $\mathcal{F}^{\fP, \fb_0}$. Since $(\fE  - \gamma \fP)$ has full rank, the projection $\operatorname{Proj}^{\fP}(\bar{\fmu})$ is continuous in $\fP$ \citep{penrose1955generalized, ding1993perturbation}. Therefore, if $\fP$ is sufficiently close to $\fP_0$, we have $\operatorname{Proj}^{\fP}(\bar{\fmu})>\bm 0$ and $\fPsi^\top \operatorname{Proj}^{\fP}(\bar{\fmu})< \fb_0$ i.e. $\operatorname{Proj}^{\fP}(\bar{\fmu})\in \relint\mathcal{F}^{\fP, \fb_0}$. Hence, there exists some neighborhood $\mathcal{O}'_{\fP_0}\subseteq\mathcal{O}_{\fP_0}$ of $\fP_0$ such that $\relint\mathcal{F}^{\fP, \fb_0}\neq \emptyset$ for all $\fP\in\mathcal{O}'_{\fP_0}\cap\mathfrak{P}$.

    \textbf{Step 2:}\\
    Since Assumption~\ref{ass:strict_convexity} (strict convexity) holds, all sets $\RL^{\fP, \fb}(\fr)$ are singleton for any $\fr, \fP, \fb$. Throughout this proof we therefore interpret $\RL^{\fP, \fb}$ as a single-valued mapping and write $\fmu = \RL^{\fP, \fb}(\fr)$ instead of $\bc{\fmu}=\RL^{\fP, \fb}(\fr)$. 

    Now, let $\mathcal{P}'\defeq \mathcal{O}'_{\fP_0}\cap\mathfrak{P}\subseteq\mathcal{P}$ and $\mathcal{B}'\defeq\bc{\fb\in\mathcal{B}: \fb \geq \fb_0}\subseteq\mathcal{B}$. Due to Assumption~\ref{ass:realizability} (realizability) and Proposition~\ref{prop:consistency} we have $\frE\in \IRL^{\fP_0, \fb_0}(\fmuE)$. Hence, the following chain of implications holds:
    \begin{align}\label{app:eq:gen_reward_intersection}
        &\RL^{\fP, \fb}(\fr) = \RL^{\fP, \fb}(\fr'), \;\forall \fr,\fr'\in \IRL^{\fP_0, \fb_0}(\fmuE), \forall\fP\in\mathcal{P}, \forall\fb\in\mathcal{B}\\
        \stackrel{(i)}{\iff} &\RL^{\fP, \fb}(\fr) = \RL^{\fP, \fb}(\frE), \;\forall \fr\in \IRL^{\fP_0, \fb_0}(\fmuE), \forall\fP\in\mathcal{P}, \forall\fb\in\mathcal{B}\nonumber\\
        \stackrel{(ii)}{\implies} &\RL^{\fP, \fb}(\fr) = \RL^{\fP, \fb}(\frE), \;\forall \fr\in \IRL^{\fP_0, \fb_0}(\fmuE), \forall\fP\in\mathcal{P}', \forall\fb\in\mathcal{B}'\nonumber\\
        \stackrel{(iii)}{\iff} &\IRL^{\fP_0, \fb_0}(\fmuE)\subseteq \bs{\partial f(\fmu) + \mathcal{U}^{\fP} + \mathcal{C}^{\fb}(\fmu) + \mathcal{E}(\fmu)}_{\fmu = \RL^{\fP, \fb}(\frE)}, \;\forall\fP\in\mathcal{P}', \forall\fb\in\mathcal{B}'\nonumber\\
        \stackrel{(iv)}{\iff} &\IRL^{\fP_0, \fb_0}(\fmuE)\subseteq\bigcap_{\fP\in\mathcal{P}'} \bigcap_{\fb\in \mathcal{B}'}  \bs{\partial f(\fmu) + \mathcal{U}^{\fP} + \mathcal{C}^{\fb}(\fmu) + \mathcal{E}(\fmu)}_{\fmu = \RL^{\fP, \fb}(\frE)}\nonumber\\
        \stackrel{(v)}{\iff} &\IRL^{\fP_0, \fb_0}(\fmuE)\subseteq\bigcap_{\fP\in\mathcal{P}'} \bigcap_{\fb\in \mathcal{B}'}  \bs{\nabla f(\fmu) + \mathcal{U}^{\fP} + \mathcal{C}^{\fb}(\fmu)}_{\fmu = \RL^{\fP, \fb}(\frE)}.\nonumber
    \end{align}
    Here, $(i)$ holds since $\frE\in \IRL^{\fP_0, \fb_0}(\fmuE)$, $(ii)$ follows from $\mathcal{P}'\subseteq\mathcal{P}$ and $\mathcal{B}'\subseteq\mathcal{B}$, and $(iii)$ is a consequence of Theorem~\ref{thm:identifiability} which applies since Assumption~\ref{ass:slater} (Slater's condition) is satisfied for all $\fP\in\mathcal{P}', \fb\in\mathcal{B}'$. Moreover, $(iv)$ follows from the definition of the intersection, and $(v)$ from differentiability and Assumption~\ref{ass:essential_smoothness} which ensures $\mathcal{E}(\fmu) = \bm 0$. 
    
    Next, we recall that $\mathcal{B} = \R^k$. Thus, we may choose a large enough $\bar{\fb}\in\mathcal{B}'$ such that the set of active safety constraints is empty for any $\fmu$ and hence $\mathcal{C}^{\bar{\fb}}(\fmu) = \bm 0$. This allows us to further simplify \eqref{app:eq:gen_reward_intersection} to:
    \begin{align}
        &\IRL^{\fP_0, \fb_0}(\fmuE)\subseteq\bigcap_{\fP\in\mathcal{P}'} \bigcap_{\fb\in \mathcal{B}'}  \bs{\nabla f(\fmu) + \mathcal{U}^{\fP} + \mathcal{C}^{\fb}(\fmu)}_{\fmu = \RL^{\fP, \fb}(\frE)}\\
        \iff &\IRL^{\fP_0, \fb_0}(\fmuE)\subseteq\bigcap_{\fP\in\mathcal{P}'}  \bs{\nabla f(\fmu) + \mathcal{U}^{\fP}}_{\fmu = \RL^{\fP, \bar{\fb}}(\frE)}\nonumber\\
        \iff &\IRL^{\fP_0, \fb_0}(\fmuE)\subseteq\bigcap_{\fP\in\mathcal{P}'}  \bs{\frE + \mathcal{U}^{\fP}} = \frE + \bigcap_{\fP\in\mathcal{P}'}  \mathcal{U}^{\fP}.\nonumber
    \end{align}
    Therefore, it suffices to show that $\bigcap_{\fP\in\mathcal{P}'}  \mathcal{U}^{\fP}\subseteq \spn\ones_{nm}$. In particular, it is enough to show that $\mathcal{U}^{\fP_1}\cap  \mathcal{U}^{\fP_2} = \spn\br{\mathbf{1}_{nm}}$ for two $\fP_1, \fP_2 \in\mathcal{P}'$. To that end, we will continue by first showing that there are $\bar{\fP}_1, \bar{\fP}_2\in\mathfrak{P}$ such that $\mathcal{U}^{\bar{\fP}_1}\cap  \mathcal{U}^{\bar{\fP}_2} = \spn\br{\mathbf{1}_{nm}}$.

    \textbf{Step 3:} Note that, as shown by \citet{rolland2022identifiability}, the condition $\mathcal{U}^{\bar{\fP}_1}\cap  \mathcal{U}^{\bar{\fP}_2} = \spn\br{\mathbf{1}_{nm}}$ is equivalent to 
\begin{equation}\label{app:eq:rank_condition1}
    \rank \myvec{\fE - \gamma \bar{\fP}_1, \fE - \gamma \bar{\fP}_2} = 2n - 1.
\end{equation}
This follows from the two facts that $(a)$ for any $\fP\in\mathfrak{P}$ we have $\bm{1}_{nm}\in\mathcal{U}^{\fP}$ and $(b)$ the condition \eqref{app:eq:rank_condition1} is equivalent to $\dim(\mathcal{U}^{\bar{\fP}_1}\cap  \mathcal{U}^{\bar{\fP}_2}) = 1$. Here, $(a)$ holds since
\begin{equation}
    (\fE - \gamma\fP)\bm{1}_{n} = \myvec{(\id_n - \gamma \fP_{a_1})\bm{1}_{n} \\ \vdots \\ (\id_n - \gamma \fP_{a_m})\bm{1}_{n}} = \myvec{(1-\gamma)\bm{1}_{n} \\ \vdots \\ (1-\gamma)\bm{1}_{n}} = (1-\gamma) \bm{1}_{nm},
\end{equation}
where we use that $\bm{1}_{n}$ is for both $\id_n$ and $\fP_{a_i},i=1,...,m$ an eigenvector to the eigenvalue $1$. Furthermore, $(b)$ is a consequence of
\begin{equation}
    \text{dim}(\spn \bm{A}_1 \cap \spn \bm{A}_2) =  (\rank \bm{A}_1  + \rank \bm{A}_2) - \rank \myvec{\bm{A}_1, \bm{A}_2},
\end{equation}
for the two matrices $\bm{A}_i = \fE - \gamma \bar{\fP}_i, i=1,2$. Therefore, the goal for this step is to prove the following claim:

\textbf{Claim} There exist $\bar{\fP}_1, \bar{\fP}_2\in\mathfrak{P}$ such that $\rank \myvec{\fE - \gamma \bar{\fP}_1, \fE - \gamma \bar{\fP}_2} = 2n - 1$.

Note that since the vector $\myvec{\bm{1}_{n}^\top, -\bm{1}_{n}^\top}^\top$ lies in the kernel of $\myvec{\fE - \gamma \bar{\fP}_1, \fE - \gamma \bar{\fP}_2}$, the rank cannot be larger than $2n-1$. Moreover, since the rank of a matrix is larger or equal than the rank of any submatrix, it suffices to prove the claim for $m=2$. To this end, let $\bar{\fP}_1, \bar{\fP}_2\in\mathfrak{P}$ be defined as follows
\begin{equation}
    \bar{\fP}_1 = \myvec{\id_n \\ \fD} ,\; \Bar{\fP}_2 = \myvec{\fD\\ \id_n},
\end{equation}
where
\begin{equation}
    \fD \defeq \myvec{0 & 1 & 0 &  \hdots & 0\\
    0 & 0 & 1 &  \hdots & 0\\
    \vdots &  \vdots & \vdots  &\ddots & \vdots\\
    0 & 0 & 0 & \hdots & 1 \\
    0 & 0 & 0 & \hdots & 1}\in\R^{n\times n}.
\end{equation}
It then holds $\rank (\id_n - \fD) = n-1$ as is readily seen since $\id_n - \fD$ is an upper triangular matrix in row echelon form. Now, in order to prove that the rank of
\begin{equation}
	C \defeq \myvec{\fE - \gamma \bar{\fP}_1, \fE - \gamma \bar{\fP}_2}\in\R^{2n\times 2n},
\end{equation}
equals $2n-1$, we show that the first $2n-1$ columns of $C$ are linearly independent. To this end, let $C^{(-2n)}\in\R^{2n\times (2n-1)}$ denote the submatrix obtained by removing the last column of $C$. Moreover, let $\fz:=\myvec{\fx^\top& \fy^\top}^\top$ with $\fx\in\R^n, \fy\in\R^{n-1}$. The columns of $C^{(-2n)}$ are linearly independent if
\begin{equation}\label{app:eq:lin_indep}
    C^{(-2n)} \fz = \bm 0  \implies \fz = \bm 0.
\end{equation}
Plugging in the definition of $\bar{\fP}_1, \bar{\fP}_2$ it holds
\begin{align}
	C^{(-2n)} \fz = \myvec{(1-\gamma)\id_n & \id_n - \gamma \fD^{(-n)}\\ \id_n - \gamma \fD & (1-\gamma) \id_n^{(-n)}} \myvec{\fx\\\fy} = \myvec{(1-\gamma)\id_n & \id_n - \gamma \fD\\ \id_n - \gamma \fD & (1-\gamma) \id_n} \myvec{\fx\\\fy\\0},
\end{align}
where we again use the notation $\bm{B}^{(-n)}$ to denote the submatrix of some matrix $\bm B$ obtained when removing the $n$-th column. Therefore, we can rewrite \eqref{app:eq:lin_indep} as 
\begin{align}
	\fx + \myvec{\fy\\ 0} - \gamma\br{\fx + \fD \myvec{\fy\\ 0}} &= \bm 0\\
	\fx + \myvec{\fy\\ 0} - \gamma\br{\fD \fx + \myvec{\fy\\ 0}} &= \bm 0.\nonumber
\end{align}
Substituting $\tilde{\fx} \defeq \fx + \myvec{\fy\\ 0}$ we get
\begin{align}
	(1-\gamma)\tilde{\fx} &= \gamma(\fD - \id_n)\myvec{\fy\\ 0}\\
	(\id_n-\gamma\fD)\tilde{\fx} &= -\gamma(\fD - \id_n)\myvec{\fy\\ 0}.\nonumber
\end{align}
This implies 
\begin{equation}
	(1-\gamma)\tilde{\fx} + (\id_n-\gamma\fD)\tilde{\fx} = 2 (\id_n - \gamma \tilde{\fD})\tilde{\fx} = \bm 0,
\end{equation}
for $\tilde{\fD}\defeq (\id_n + \fD)/2$. Since $\tilde{\fD}$ is again a row stochastic matrix, $\id_n - \gamma \tilde{\fD}$ is invertible and thus $\tilde{\fx} = \bm 0$. Moreover, since $\rank (\id_n - \fD) = n-1$ and $(\id_n - \fD)\ones_n = \bm 0$, we have $\ker (\id_n - \fD) = \spn(\ones_n)$. In light of
\begin{equation}
    (\fD - \id_n)\myvec{\fy\\ 0} = \bm 0,
\end{equation}
this implies that $\fy = \bm 0$, which proves the claim.

\textbf{Step 4:}\\
Equipped with the above claim, the final step of the proof is to show that there are $\fP_1, \fP_2 \in \mathcal{P}'\subseteq \mathcal{P}$ such that 
\begin{equation}\label{app:eq:intersection2}.
    \mathcal{U}^{\fP_1}\cap \mathcal{U}^{\fP_2} = \spn(\ones_{nm}).
\end{equation}
Analogously to the previous step, we will show that
\begin{equation}\label{app:eq:rank_condition2}
    \rank \myvec{\fE - \gamma \fP_1, \fE - \gamma \fP_2} = 2n - 1.
\end{equation}
For this purpose, we choose two arbitrary (and possibly equal) $\fP_{1,0}, \fP_{2,0}\in\mathcal{P}'$ and define
\begin{equation}
    (\fP_1(\tau), \fP_2(\tau)) \defeq (1-\tau) (\fP_{1,0}, \fP_{2,0}) + \tau (\bar{\fP}_1, \bar{\fP}_2) \in \mathfrak{P}\times \mathfrak{P}.
\end{equation}
Since $\rank \myvec{\fE - \gamma \fP_1(\tau), \fE - \gamma \fP_2(\tau)} = 2n - 1$ for $\tau = 1$, there exists a $(2n-1)\times (2n-1)$ sub-matrix $\bm{C}_{\text{sub}}(\tau)$ of $\myvec{\fE - \gamma \fP_1(\tau), \fE - \gamma \fP_2(\tau)}$ that is invertible for $\tau = 1$. Thus, the function $h(\tau)\defeq \det \bm{C}_{\text{sub}}(\tau)$
is a non-zero polynomial, which by the fundamental theorem of algebra can only have finitely many roots. Since the set $\mathcal{O}_{\fP_0}'$ is a neighborhood of $\fP_0$, it contains an open ball (in any norm\footnote{Since all norms are equivalent in finite dimensional vector spaces the choice of norm is irrelevant.}) around $\fP_0$. Furthermore, $\mathfrak{P}$ is convex. Hence, we can always choose a small enough $\varepsilon>0$ such that $\fP_1(\varepsilon), \fP_2(\varepsilon)\in \mathcal{P}'=\mathcal{O}_{\fP_0}'\cap\mathfrak{P}$ and $h(\varepsilon)\neq 0$. However, for $h(\varepsilon)\neq 0$ the rank condition \eqref{app:eq:rank_condition2} is satisfied and we have proven the equality \eqref{app:eq:intersection2}.
\end{proof}

\begin{remark} 
Note that in fact we have proven a stronger result than the equality \eqref{app:eq:intersection2} -- namely that for large enough $\fb$ the set of $(\fP_1, \fP_2)$ for which 
\begin{equation}\label{app:eq:intersection}
	\IRLO^{\fP_1, \fb}\circ\RL^{\fP_1, \fb}(\frE) \cap \IRLO^{\fP_2, \fb}\circ\RL^{\fP_2, \fb}(\frE) = \frE + \spn(\ones_{nm}),
\end{equation}
is dense in $\mathfrak{P}\times\mathfrak{P}$. \citet{cao2021identifiability,rolland2022identifiability} analyze \eqref{app:eq:intersection2} in the context of identifiability from two experts. They show experimentally that the rank condition \eqref{app:eq:rank_condition2} is always satisfied when randomly generating two transition laws $(\fP_1, \fP_2)$. However, the formal proof that the set where \eqref{app:eq:rank_condition2} is satisfied is dense in $\mathfrak{P}\times\mathfrak{P}$ is novel.
\end{remark}


\subsection{Proof of Proposition~\ref{prop:linear_reward_identifiability}}\label{app:subsec:cor:linear_reward_identifiability}
\noindent\textbf{Proposition~\ref{prop:linear_reward_identifiability}}
\textit{
    Let Assumption~\ref{ass:slater}, \ref{ass:strict_convexity}, \ref{ass:realizability}, \ref{ass:essential_smoothness} hold with $\fmuE\in\RL(\frE)$ for some $\frE\in\mathcal{R}$ and
    \begin{equation}\label{app:eq:linear_rewardclass}
        \mathcal{R} \subseteq \bc{\fr_{\fw} = \fPhi \fw: \fPhi \in\R^{mn\times d}, \fw\in\R^d}.
    \end{equation}
    Then, if for $\fXi\defeq\myvec{\fE - \gamma \fP, \fPsi}$ it holds that
    \begin{equation}\label{app:eq:cirl_identifiability_condition}
        \rank \myvec{\fPhi, \fXi} - \br{\rank \fPhi + \rank \fXi} = 0,
    \end{equation}
    then we have $\IRL(\fmuE) = \bc{\frE}$.
}
\begin{proof}
First, we note that the rank condition is equivalent to the condition that the subspace spanned by the reward features intersects the Minkowski sum of the subspace of potential shaping transformations and the subspace spanned by the safety constraints, only at zero:
\begin{align}
	&\spn\fPhi\cap \br{\mathcal{U}+\spn\fPsi}=\bm 0\\
	\iff &\dim\br{ \spn\fPhi\cap \br{\mathcal{U}+\spn\fPsi}} = 0\nonumber\\
	\iff &\dim\br{\spn\fPhi} + \dim\br{\spn \fXi} - \dim \br{\spn \myvec{\fPhi, \fXi}}=0\nonumber\\
	\iff &\rank \myvec{\fPhi, \fXi} - \br{\rank \fPhi + \rank \fXi} = 0.\nonumber
\end{align}
To ease notation we define $\mathcal{V}\defeq \spn \fPsi$ and $\mathcal{W}\defeq \spn \fPhi$. Observe that $\mathcal{U}, \mathcal{V}, \mathcal{W}$ are linear subspaces of $\R^{nm}$. By Proposition~\ref{prop:consistency}, we have $\frE\in\IRL(\fmuE)$. Furthermore, due to \eqref{app:eq:cirl_identifiability_condition} it holds $\br{\mathcal{U} + \mathcal{V}}\cap\mathcal{W} = \bm 0$. We therefore get
    \begin{align*}
        \IRL(\fmuE) &\stackrel{(i)}{=} \br{\nabla f(\fmuE) + \mathcal{U} + \mathcal{C}}\cap\mathcal{R}\numberthis \\
        &\stackrel{(ii)}{\subseteq}\br{\nabla f(\fmuE) + \mathcal{U} + \mathcal{C}}\cap\mathcal{W}\\
        &\stackrel{(iii)}{\subseteq} \br{\nabla f(\fmuE) + \mathcal{U} + \mathcal{V}}\cap\mathcal{W}\\
        &\stackrel{(iv)}{=} \br{\frE + \mathcal{U} + \mathcal{V}}\cap\mathcal{W}\\
        &\stackrel{(v)}{=}\frE + \br{\mathcal{U} + \mathcal{V}}\cap\mathcal{W}\\
        &\stackrel{(vi)}{=} \frE
    \end{align*}
    Here, we used Theorem~\ref{thm:identifiability} in $(i)$, the reward class \eqref{app:eq:linear_rewardclass} in $(ii)$, and the inclusion $(iii)$ holds since $\mathcal{C}\subset\mathcal{V}$. Furthermore, $(iv)$ and $(v)$ follow from $\frE\in\IRL(\fmuE)$, and $(vi)$ from \eqref{app:eq:cirl_identifiability_condition}. This concludes the proof.
\end{proof}

\section{Proof of Theorem~\ref{thm:sample_complexity}}\label{app:sec:sample_complexity}
\textbf{Theorem}~\ref{thm:sample_complexity}
\textit{
          Let Assumption~\ref{ass:realizability} and let $\fmuE\in\RL(\frE)$ for some $\frE\in\mathcal{R}\defeq\mathcal{R}^{\norm{\cdot}_1}$. Let $\fmuhat \in \RL\circ\IRL(\fmuEhat)$ and $R\defeq\max_{s,a}\norm{\fPhi(s,a)}_{\infty}$. Choosing 
     \begin{equation}
         N = \ceil{\dfrac{32R^2}{\varepsilon^2}\log\br{\dfrac{2d}{\delta}}} \text{ and } T=\ceil{\log\br{\dfrac{\varepsilon}{8R}}/\log(\gamma)},
     \end{equation}
     it holds with probability at least $1-\delta$
    \begin{align}\label{app:eq:value_diff}
        &J(\fmuE, \frE) - J(\fmuhat, \frE) \leq \varepsilon,\\ 
        &J(\fmuhat, \frhat) - J(\fmuE, \frhat) \leq \varepsilon,\nonumber\\
        &\forall \frhat \in \IRL(\fmuEhat),\nonumber
    \end{align}
    where $J(\fmu,\fr)\defeq \fr^\top\fmu - f(\fmu)$. Moreover, if\vspace{-0.2cm}
    \begin{enumerate}[(a)]
        \item $f$ is $L$-strongly convex with respect to the norm $\norm{\cdot}$, it holds with probability at least $1-\delta$
        \begin{equation}\label{app:eq:occ_diff}
            \norm{\fmuhat - \fmuE} \leq \sqrt{\dfrac{2\varepsilon}{L}}.
        \end{equation}
        \vspace{-0.4cm}
        \item $f(\fmu)=-\beta \,\E_{(s,a)\sim\fmu}\bs{H\left(\fpi^{\fmu}(\cdot|s)\right)}$ with $\beta>0$, it holds with probability at least $1-\delta$
        \begin{equation}\label{app:eq:policy_diff}
         \E_{(s,a)\sim\fmuE}\bs{\norm{\fpi^{\fmuhat}(\cdot|s)-\fpiE(\cdot|s)}_1}\leq \sqrt{\dfrac{2\varepsilon}{\beta}}.
        \end{equation}
    \end{enumerate}
}

\begin{proof}
Consider the idealized and the empirical min-max objective $L(\fmu, \fw)\defeq\fr_{\fw}^\top\br{\fmu-\fmuE}-f(\fmu)$ and $\Lhat(\fmu, \fw)\defeq\fr_{\fw}^\top\br{\fmu-\fmuEhat}-f(\fmu)$, the main idea of the proof it to get a uniform bound on $\abs{L-\Lhat}$. The statement in \eqref{app:eq:value_diff} is then following from the saddle point property of $(\fmuE, \frE)$ and $(\fmuhat, \frhat)$. For \eqref{app:eq:occ_diff} (and \eqref{app:eq:policy_diff}) we are then using strict concavity (and the soft-suboptimality Lemma) to translate proximity of the optimal value to proximity of the optimal occupancy (and the optimal policy).
\\\\
\textbf{Step 1:}\\
First, we define the true expert feature expectations $\fsigmaE \defeq \fPhi^\top\fmuE$, the empirical expert feature expectation $\fsigmaEhat\defeq \fPhi^\top\fmuEhat$, as well as $\fsigmaET\defeq (1-\gamma)\E\bs{\sum_{t=0}^T \gamma^t \fPhi(s_t, a_t)\big| \piE}$. We can then decompose $\max\abs{\hat{L}-L}$ as follows:
\begin{align}
    &\max_{\fmu\in\mathcal{F}, \norm{\fw}_1\leq 1} \abs{L(\fmu,\fw)-\Lhat(\fmu,\fw)} = \max_{\norm{\fw}_1\leq 1} \abs{\fw^\top\br{\fsigmaEhat-\fsigmaE}} \\
    &\stackrel{(i)}{\leq} \max_{\norm{\fw}_1\leq 1} \abs{\fw^\top\br{\fsigmaE-\fsigmaET}} + \max_{\norm{\fw}_1\leq 1} \abs{\fw^\top\br{\fsigmaET-\fsigmaEhat}}\nonumber\\
    &\stackrel{(ii)}{\leq} \underbrace{\norm{\fsigmaE-\fsigmaET}_\infty}_{I_1}+ \underbrace{\norm{\fsigmaET-\fsigmaEhat}_\infty}_{I_2} \nonumber
\end{align}
Here, $(i)$ follows from the triangle inequality for the supremum norm and $(ii)$ from Hölder's inequality $|\f{x}^\top\f{y}|\leq \norm{\f{x}}_1\norm{\f{y}}_\infty$.
The first term is readily bounded by
\begin{equation}
    I_1 =  \norm{(1-\gamma)\E_\pi \sum_{t=T+1}^{\infty}\gamma^t\fPhi(s_t,a_t)}_{\infty} \leq  \gamma^{T+1} R \leq \gamma^T R,
\end{equation}
where $R \defeq \max_{s,a} \norm{\fPhi(s,a)}_{\infty} = \max_{s,a}\max_{\norm{\fw}_1 \leq 1} \abs{\fr_{\fw}(s,a)}$. Thus in order to have $I_1 \leq \varepsilon_1$ it suffices to choose $T=\ceil{\log(\varepsilon_1 /R)/\log(\gamma)}$. 
\\\\
For the second term $I_2$, we make use of Hoeffding's inequality
\begin{lemma}[Hoeffding]
Consider iid random variables $X_1,\hdots,X_N$ with $X_i\in[a,b]$ and let $\bar{X}_N:=\frac{1}{N}(X_1+\hdots+X_N)$. Then,
\begin{equation}
    \Pr\br{\abs{\bar{X}_N - \E X_i} \geq t}\leq 2\exp\br{-\dfrac{2t^2 N}{(b-a)^2}}.
\end{equation}
\end{lemma}
Since $\abs{\br{\fsigmaEhat}_j}\leq R$ and $\E\br{\fsigmaEhat}_j=\br{\fsigmaET}_j$ for all $j=1,\hdots,d$, we get by Hoeffding's inequality
\begin{equation}
    \Pr\br{\abs{\br{\fsigmaEhat- \fsigmaET}_j}\geq \varepsilon_2}\leq 2\exp\br{-\dfrac{\varepsilon_2^2 N}{2R^2}},\quad j=1,\hdots,d.
\end{equation}
Using the union bound
\begin{align}
    \Pr\br{I_2 < \epsilon_2} &= 1 - \Pr\br{I_2 \geq \epsilon_2} \\&\geq
   1- \br{\Pr\br{\abs{\br{\fsigmaEhat- \fsigmaET}_1}\geq \varepsilon_2} + \hdots + \Pr\br{\abs{\br{\fsigmaEhat- \fsigmaET}_d}\geq \varepsilon_2}},\nonumber
\end{align}
yields $I_2\leq \varepsilon_2$ with probability at least $1-\delta$ when $N \geq \log\br{2d/\delta}2R^2/\varepsilon_2^2$. Thus choosing $N = \ceil{\log\br{2d/\delta}32R^2/\varepsilon^2}$ and $T=\ceil{\log(\varepsilon/(8R))/\log(\gamma)}$ it holds with probability at least $1-\delta$
\begin{equation}\label{app:eq:uniform_bound_f}
    \max_{\fmu\in\mathcal{F}, \norm{\fw}_1\leq 1} \abs{L(\fmu,\fw)-\Lhat(\fmu,\fw)} \leq \varepsilon/2.
\end{equation}
\\\\
\textbf{Step 2:}\\
Next we use the fact that for two real-valued functions $g$ and $h$ we always have $\abs{\max g - \max h}\leq\max \abs{g-h}$ and similarly $\abs{\min g - \min h}\leq \max \abs{g-h}$. Therefore, it holds for all $\frhat\in\IRL(\fmuEhat)$ with probability at least $1-\delta$:
\begin{align}\label{app:eq:saddle_diff}
    \abs{L(\fmuE, \frE) - \Lhat(\fmuhat, \frhat)} &= \abs{\min_{\norm{\fw}_1\leq 1}\max_{\fmu\in\mathcal{F}} L(\fmu, \fw)-\min_{\norm{\fw}_1\leq 1}\max_{\fmu\in\mathcal{F}} \Lhat(\fmu, \fw)}\\
    &\leq \max_{\norm{\fw}_1\leq 1}\abs{\max_{\fmu\in\mathcal{F}} L(\fmu, \fw)-\max_{\fmu\in\mathcal{F}} \Lhat(\fmu, \fw)}\nonumber\\
    &\leq \max_{\fmu\in\mathcal{F}, \norm{\fw}_1\leq 1} \abs{\Lhat(\fmu,\fw)-L(\fmu,\fw)} \leq \varepsilon/2.\nonumber
\end{align}
Due to Assumption~\ref{ass:realizability} there is $\fw^*$ with $\norm{\fw^*}_1\leq 1$ such that $\fr_{\fw^*}=\frE$. Furthermore, let $\fwhat$ with $\norm{\fwhat}_1\leq 1$ be such that $\fr_{\fwhat} \in \IRL(\fmuEhat)$. Then, $L$ has a saddle-point in $(\fmuE, \fw^*)$ and $\Lhat$ in $(\fmuhat, \fwhat)$. Choosing $N$ and $T$ as in \eqref{app:eq:uniform_bound_f}, it holds with probability at least $1-\delta$
\begin{align}
    J(\fmuE, \frE) - J(\fmuhat, \frE) &= L(\fmuE, \fw^*) - L(\fmuhat, \fw^*)\\
    &\stackrel{(i)}{\leq} L(\fmuE, \fw^*) - \Lhat(\fmuhat, \fw^*) + \varepsilon/2\nonumber\\
    &\stackrel{(ii)}{\leq} L(\fmuE, \frE) - \Lhat(\fmuhat, \fwhat) + \varepsilon/2\nonumber\\
    &\stackrel{(iii)}{\leq} \varepsilon,\nonumber
\end{align}
where $(i)$ is a consequence of \eqref{app:eq:uniform_bound_f}, inequality $(ii)$ follows since $\Lhat$ has a saddle point in $(\fmuhat, \fwhat)$, and $(iii)$ from \eqref{app:eq:saddle_diff}. This proves the first result in \eqref{app:eq:value_diff}. The second follows analogously from a similar series of inequalities
\begin{align}
    J(\fmuhat, \frhat) - J(\fmuE, \frhat) &= L(\fmuhat, \fwhat) - L(\fmuE, \fwhat) \\
    &\leq \Lhat(\fmuhat, \fwhat) - L(\fmuE, \fwhat) + \varepsilon/2\nonumber\\
    &\leq \Lhat(\fmuhat, \fwhat) - L(\fmuE, \fw^*) + \varepsilon/2\nonumber\\
    &\leq \varepsilon.\nonumber
\end{align}
\\\\
\textbf{Step 3:} \\
To prove the inequality \eqref{app:eq:occ_diff}, we use the fact that since $f$ is $L$-strongly convex, $J(\fmu, \frE) = {\frE}^{\top} \fmu - f(\fmu_k)$ is $L$-strongly concave in $\fmu$ i.e. it holds
\begin{equation}
    J(\fmu, \frE) \leq J(\fmuE, \frE) + \nabla_{\fmu} J(\fmuE, \frE)^\top \br{\fmu-\fmuE} - \dfrac{L}{2}\norm{\fmu - \fmuE}^2.
\end{equation}
By optimality of $\fmuE$ it holds $\nabla_{\fmu} J(\fmuE, \frE)^\top \br{\fmu-\fmuE}\leq 0$ for all $\fmu\in\mathcal{F}$. Rearranging terms yields and taking the square root yields
\begin{equation}
    \norm{\fmu- \fmuE} \leq \sqrt{\dfrac{2}{L}\br{J(\fmuE, \frE) - J(\fmu, \frE)}}, \quad\forall \mu\in\mathcal{F}.
\end{equation}
Combining this with \eqref{app:eq:value_diff} yields the desired result
\begin{equation}
    \norm{\fmuhat- \fmuE} \leq \sqrt{\dfrac{2\varepsilon}{L}}.
\end{equation}
Although $f(\fmu)=-\beta \,\E_{(s,a)\sim\fmu}\bs{H\left(\fpi^{\fmu}(\cdot|s)\right)}$ is in general not strongly convex (although it is strictly convex), we can make use of the following result for entropy-regularized MDPs.
\begin{lemma}[\citep{mei2020global}]\label{app:lemma:soft-sub-optimality}
    For any occupancy measure $\fmu\in\mathcal{F}$ it holds
    $$J(\fmu^*, \fr) - J(\fmu, \fr) = \beta \sum_s \fnu(s) \DKL\br{\fpi^{\fmu}(\cdot|s)||\fpi^{\fmu^*}(\cdot|s)},  $$
    where $\fmu^*\in \RL(\fr)$ and $\fnu(s) = \sum_a \fmu(s,a)$ is the state occupancy measure.
\end{lemma}
Making use of the following lower bound on the KL-divergence \citep{cover1999elements}
\begin{equation}
    \DKL(\fq||\fp)\leq \dfrac{1}{2}\norm{\fq-\fp}_1^2,
\end{equation}
we arrive at
\begin{align}
    J(\fmuhat, \frhat) - J(\fmuE, \frhat) &\geq \beta \sum_s \fnu^{\text{E}}(s) \DKL\br{\fpi^{\fmuE}(\cdot|s)||\fpi^{\fmuhat}(\cdot|s)}\nonumber\\
    &\geq \dfrac{\beta}{2} \sum_s \fnu^{\text{E}}(s) \norm{\fpi^{\fmuE}(\cdot|s)-\fpi^{\fmuhat}(\cdot|s)}_1^2\nonumber\\
    &\geq \dfrac{\beta}{2} \br{\sum_s \fnu^{\text{E}}(s) \norm{\fpi^{\fmuE}(\cdot|s)-\fpi^{\fmuhat}(\cdot|s)}_1}^2\nonumber\\
    &= \dfrac{\beta}{2} \br{\E_{(s,a)\sim\fmuE}\bs{\norm{\fpi^{\fmuE}(\cdot|s)-\fpi^{\fmuhat}(\cdot|s)}_1}}^2,
\end{align}
where the last inequality follows from Jensen's inequality. Making use of \eqref{app:eq:value_diff} and rearranging terms yields
\begin{equation}
    \E_{(s,a)\sim\fmuE}\bs{\norm{\fpi^{\fmuE}(\cdot|s)-\fpi^{\fmuhat}(\cdot|s)}_1}\leq \sqrt{\dfrac{2\varepsilon}{\beta}}.
\end{equation}
\end{proof}


\section{Algorithm}\label{app:sec:algorithm}
We present the policy based algorithm for entropy regularization as used in the experiments. To this end, recall the min-max problem \eqref{eq:cirl_occ_emp}
\begin{equation}
    \min_{\fr\in\mathcal{R}}\max_{\fmu\in\mathcal{F}}  \;\;\fr^\top\br{\fmu - \fmuEhat} - f(\fmu),
\end{equation}
and the entropy regularization $f(\fmu)=-\beta \,\E_{(s,a)\sim\fmu}\bs{H\left(\fpi^{\fmu}(\cdot|s)\right)}$. Applying Proposition~\ref{prop:strong_duality} this is equivalent to
\begin{align}\label{app:eq:dual1}
    &\min_{\fxi\geq\f{0}, \fr\in\mathcal{R}} \max_{\fmu\in\mathcal{M}}
    \fr^\top\br{\fmu - \fmuEhat} - f(\fmu) + \fxi^\top \br{\fb - \fPsi^\top \fmu}.
\end{align}
Using the one-to-one mapping between policies and occupancy measures and the linear reward class
\begin{align}\label{app:eq:reward_classes}
    \mathcal{R} \defeq\bc{\fr_{\fw}=\fPhi \fw: \fPhi\in\R^{nm\times d}, \norm{\fw}\leq c},
\end{align}
this can be rewritten as 
\begin{align}
    &\min_{\fxi\geq\f{0}, \fr\in\mathcal{R}} \max_{\fpi\in\Pi} \fr^\top\br{\fmu^{\fpi} - \fmuEhat} - f(\fmu^{\fpi}) + \fxi^\top \br{\fb - \fPsi^\top \fmu^{\fpi}}\\
    =&\min_{\fxi\geq\f{0}, \norm{\fw}\leq c}\max_{\fpi\in\Pi} \fw^\top \fPhi^\top\br{\fmu^{\fpi} - \fmuEhat} - f(\fmu^{\fpi}) + \fxi^\top \br{\fb - \fPsi^\top \fmu^{\fpi}}\nonumber\\
    =&\min_{\fxi\geq\f{0}, \norm{\fw}\leq c}\max_{\fpi\in\Pi} L_{\mathcal{D}}(\fpi, \fw, \fxi),\nonumber
\end{align}
with the Lagrangian $L_{\mathcal{D}}(\fpi, \fw, \fxi) \defeq \fw^\top \fPhi^\top\br{\fmu^{\fpi} - \fmuEhat} - f(\fmu^{\fpi}) + \fxi^\top \br{\fb - \fPsi^\top \fmu^{\fpi}}$. Motivated by recent advances in min-max optimization \citep{daskalakis2018limit}, we suggest to use a gradient descent-ascent method, where policy and reward are updated simultaneously within a single optimization loop.

\begin{algorithm}[H]
   \caption{Gradient Descent Ascent for Constrained Entropy-Regularized IRL}
   \label{alg:1}
\begin{algorithmic}
   \STATE {\bfseries Input:} Expert data $\mathcal{D}$, learning rate $\eta$.
   \STATE Initialize $\fpi\in\Pi, \fw = \bm 0, \fxi= \bm 0$.
   \FOR{$i=1$ {\bfseries to} $N_{\text{episodes}}$}
    \item $\fr \leftarrow \fPhi\fw - \fPsi\fxi$
    \item $\fpi \leftarrow$ NPG$(\fpi, \fr, \eta)$
    \item $\fw \leftarrow  P_{B_c}\br{\fw - \eta \nabla_{\fw}L_{\mathcal{D}}(\fpi, \fw, \fxi)}$
    \item $\fxi \leftarrow  P_{[\bm 0, \bm \infty)}\br{\fxi- \eta \nabla_{\fxi}L_{\mathcal{D}}(\fpi, \fw, \fxi)}$
   \ENDFOR
   \STATE {\bfseries Return:} $\fw, \fpi$.
\end{algorithmic}
\end{algorithm}
Here, NPG$(\fpi, \fr, \eta)$ refers to a single entropy-regularized Natural Policy Gradient step with softmax parametrization of the policy \citep{cen2022fast}. Note that for the step size $\eta = (1-\gamma)/\beta$ this policy gradient step is completely equivalent to soft policy iteration \citep{haarnoja2018soft}. Furthermore, the gradients for $\fw$ and $\fxi$ are given by
\begin{align}
    \nabla_{\fw}L_{\mathcal{D}}(\fpi, \fw, \fxi) &= \fPhi^\top\br{\fmu^{\fpi} - \fmuEhat}  \\
    \nabla_{\fxi}L_{\mathcal{D}}(\fpi, \fw, \fxi) &= \br{\fb - \fPsi^\top \fmu^{\fpi}} \nonumber.
\end{align}
Moreover, $P_{B_c}$ denotes the projection onto the ball $B_c \defeq \bc{\fw: \norm{\fw}\leq c}$ and $P_{[\bm 0, \bm \infty)}$ the trivial projection onto the non-negative orthant. For the $1$-norm projection we use the efficient projection algorithm by \citet{duchi2008efficient} and its implementation by \citet{ong2019sigpy}. Algorithm~\ref{alg:1} has been shown to provably converge for CMDPs \citep{ying2022dual, ding2022convergence} (with known reward) and IRL \citep{zeng2022maximum} (without safety constraints).

Finally, we note that in our code we also provide a gradient descent ascent primal dual method that directly optimizes the min-max problem \eqref{app:eq:dual1} in the occupancy measure space. Since the problem is convex-concave such algorithms enjoy provable convergence guarantees \citep{daskalakis2018limit, nemirovski2004prox}, and work for arbitrary convex regularizations $f$. As the focus of this paper is not on algorithmic convergence and Algorithm~\ref{alg:1} was both -- easier to tune and converged more efficiently -- we used Algorithm~\ref{alg:1} throughout the presented experiments.

\end{document}